%% file: main.tex
\documentclass{llncs}
\pdfoutput=1
\usepackage[english]{babel}
\usepackage{latexsym}
\usepackage[leqno]{amsmath}
\usepackage{amssymb}
\usepackage{mathptmx}
\usepackage{xspace}
\usepackage{enumerate}
\usepackage{lineno}
\usepackage{tikz}
\usetikzlibrary{shapes,snakes}
\usepackage[labelfont=bf]{caption}[2004/07/16]
\usepackage{subfig}
\usepackage{booktabs}
\usepackage{stackrel}
\usepackage{hyperref} 
\usepackage[algo2e,lined,linesnumbered,titlenumbered,longend]{algorithm2e}

\newtheorem{thm}{Theorem}
\newtheorem{cor}[thm]{Corollary}
\newtheorem{lem}[thm]{Lemma}
\newtheorem{prop}[thm]{Proposition}
\newtheorem{defn}[thm]{Definition}

\numberwithin{equation}{section}

\newcommand{\To}{\Rightarrow}
\newcommand{\non}{\mathord{\sim}}
\newcommand{\defeater}{\leadsto}

\newcommand{\set}[2][\relax]{\ensuremath{#1\{#2#1\}}}

\newif\ifpdf
    \ifx\pdfoutput\undefined
    \pdffalse 
    \else
    \pdfoutput=1 
    \pdftrue
\fi

\ifpdf
\fi

\begin{document}

\title{Revision of Defeasible Logic Preferences}

\author{Guido Governatori$^{\bullet}$, Francesco Olivieri$^{\dag \bullet *}$,\\ Simone Scannapieco$^{\dag \bullet *}$\and Matteo Cristani$^{\dag}$}
\institute{$^{\dag}$Department of Computer Science, University of Verona, Italy\\ $^{\bullet}$NICTA, Queensland Research Laboratory, Australia\\ $^{*}$Institute for Integrated and Intelligent Systems, Griffith University, Australia }

\linenumbers

\maketitle

\begin{abstract}
There are several contexts of non-monotonic reasoning where a priority between
rules is established whose purpose is preventing conflicts.

One formalism that has been widely employed for non-monotonic reasoning is the
sceptical one known as Defeasible Logic. In Defeasible Logic the tool used for
conflict resolution is a preference relation between rules, that establishes
the priority among them.

In this paper we investigate how to modify such a preference relation in a
defeasible logic theory in order to change the conclusions of the theory
itself. We argue that the approach we adopt is applicable to legal reasoning
where users, in general, cannot change facts or rules, but can propose their
preferences about the relative strength of the rules.

We provide a comprehensive study of the possible combinatorial cases
and we identify and analyse the cases where the revision process is
successful.

After this analysis, we identify three revision/update operators and study
them against the AGM postulates for belief revision operators, to discover
that only a part of these postulates are satisfied by the three operators.

\end{abstract}

\input{IntroductionPRDL}

\input{DefLogPRDL}

\input{PrefDefRevisionPRDL}

\input{AGMComparisonPRDL}

\input{RelatedWork}

\input{concl}

\section*{Acknowledgements} 

A previous version of the paper has been presented at the 13th International
Workshop on Non-monotonic Reasoning (NMR 2010) \cite{nmr10:preferences} and to
4th International Web Rule Symposium (RuleML 2010) \cite{ruleml10preferences}.
We thank the anonymous referees for NMR 2010 and RuleML 2010 for their
valuable comments and feedback.

NICTA is funded by the Australian Government as represented by the
Department of Broadband, Communications and the Digital Economy,
the Australian Research Council through the ICT Centre of
Excellence program and the Queensland Government.

\bibliographystyle{splncs}
\bibliography{biblioPRDL}
\end{document}

%% file: IntroductionPRDL.tex
\section{Introduction} 
\label{sec:Introduction}

A large number of real-life cases in legal reasoning, information security,
digital forensics, and even engineering or medical diagnosis, exhibit the two
following circumstances: (a) different persons have different preferences, and
(b) decision making depends upon the order the rules are applied. When the
decision mechanism is based on rules, and the rules are in conflict, then
inconsistencies may be generated, and decision making may require preferences
to solve/avoid conflicts. It may occur that, under certain circumstances,
using a particular set of preferences to solve a conflict does not result in
the desired/expected outcome. Accordingly, to revise the outcome we can revise
the underlying preferences.

Non-monotonic reasoning has been advanced for common-sense reasoning and
reasoning with partial and conflicting information. We can distinguish two
types of non-monotonic reasoning: credulous and sceptical. In credulous
non-monotonic reasoning, once a conflict arises, we independently explore the
two branches of the conflict, while in sceptical non-monotonic reasoning a
conflict must be solved before proceeding with the reasoning. Here we
concentrate on sceptical non-monotonic reasoning.

Typically sceptical non-monotonic formalisms are equipped with techniques to
address conflicts, where a conflict is a combination of reasoning chains
leading to a contradiction. The most common device to handle conflicts is a
preference or superiority relation over the elements used by the formalism to
reason. These elements can be formulae, axioms, rules or arguments, and the
preference relation states that one of such elements is to be preferred to
another one when both can be used.

In this research we focus on a specific rule-based non-monotonic formalism,
Defeasible Logic, but the motivation behind the particular technical
development applies in general to other rule-based non-monotonic formalisms
equipped with a priority about rules. Indeed, considering a rule based
formalism, knowledge is partitioned in \emph{facts} (describing immutable
propositions/statements about a case), \emph{rules} (describing relationships
between a set of premises and a conclusion), and a \emph{preference relation}
or \emph{superiority relation} (describing the relative strength of rules). A
revision operation\footnote{In general we will use the term \emph{revision
operation} to denote any operation that changes a theory. In Sections
\ref{sec:PrefDefRevisionPRDL} and \ref{sec:AGMComparisonPRDL} the term will be
understood in a specific technical sense.} transforms a theory by changing
some of its elements, be it the facts, rules, or superiority relation. Revision
based on change of facts corresponds to an update operation
\cite{KatsunoMendelzon}, revision based on modification of rules has been
investigated in \cite{ki99}, whilst to the best of our knowledge, revision of
non-monotonic theories based on modifications of the underlying superiority
relation has been neglected so far.

\medskip

In this paper we study revision of defeasible theories operating on the
superiority relation. We begin by arguing that, while little attention has
been dedicated to this topic, it has natural correspondences to reasoning
patterns in legal reasoning. After that, we investigate the type of operations
that are possible for this kind of revision.

\medskip

Once we introduced the operation, it is only natural to follow by establishing
which properties they enjoy. The properties that are significant for revision
operations in belief revision in order to be rational have already been
isolated in a systematic view proposed by Alchourr{\'o}n, G{\"a}rdenfors and
Makinson \cite{AGM}, namely the AGM postulates.


The AGM postulates where designed with classical logic in mind. Classical
logic is monotonic, therefore if we add new information which is
\emph{incompatible} with the old one, an inconsistency arises. In this
scenario, the sole way to recover consistency is to invoke a revision
operator. 


Conversely, due to the nature of non-monotonic reasoning, adding new
``incompatible'' information in a non-monotonic system usually does not
generate a contradiction within the theory, even if the result may not be
conceptually satisfactory.

Given the difference in nature between classical logic and non-monotonic
reasoning, it is of interest to investigate which AGM postulates apply to
non-monotonic reasoning, to what extent, and in which form they apply.

In the recent years, a few works addressed the issue of belief revision in
non-monoto\-nic logics where there is a general understanding that the
AGM postulates are not fully appropriate for non-monotonic reasoning. For
example, \cite{GovRot:igpl09} shows that belief revision methodologies are not
suitable to changes in specific significant non-monotonic theories, and
that it is possible to revise such theories fully satisfying the AGM
postulates, but then the outcome is utterly meaningless for their purposes.

Still, the matter whether and which postulates hold is far from being settled.
For example, \cite{DelgrandeSTW08} proposed an approach to belief revision of
logic programs under answer set semantics that is fully compliant with the
base AGM postulates for revision. On the other hand, Delgrande \cite[p.
568]{Delgrande10} asserts that a subset of postulates for belief revision is
not appropriate for belief revision of non-monotonic theories (and thus is
ignored in his work), while we will argue in
Section~\ref{sec:AGMComparisonPRDL} that the same postulates can be
adopted(Alternativa: are meaningful) in our approach. This suggests that the
suitability of AGM postulates to a belief revision approach for non-monotonic
reasoning is still debatable.

\vspace{2mm}

The paper is organised as follows: In Section~\ref{sec:Norms} we motivate that
reasoning over preferences on rules and on how to modify the preferences is a
natural reasoning pattern in legal reasoning. Then, in
Section~\ref{sec:DefLogPRDL} we introduce Defeasible Logic, the formalism
chosen for our investigation; in particular, we introduce new auxiliary proof
tags to describe derivations in Defeasible Logic. The new proof tags do not
modify the expressive power of the logic, but they identify patterns where
instances of the superiority relation contribute to the derivation of a
conclusion. Armed with this technical machinery, in
Section~\ref{sec:PrefDefRevisionPRDL} we start by proving that the problem of
revising a theory changing the superiority relation is, in general, an
NP-complete problem; secondly, we provide an exhaustive analysis of the cases
and conditions under which revision operations modifying only the superiority
relation are successful. Section~\ref{sec:AGMComparisonPRDL} analyses the AGM
postulates against the introduced operators. Section~\ref{sec:related_work}
overviews closely related approaches, and Section~\ref{sec:Conclusions}
concludes the paper with a summary of the achieved results, discussion of
related works and quick hints for future developments.

\section{Norms and Preferences in Legal Reasoning} 
\label{sec:Norms}

It has been argued \cite{ruleml09:rules_and_norms} that some aspects of legal
reasoning can be captured by non-monotonic rule-based formalisms. The main
intuition is that norms can be represented by rules, the evidence in cases by
facts, and that the superiority relation is induced by legal principles
determining how to solve conflicts between norms.

We take the stance that, typically in the legal reasoning domain, we do not
have control over the rules (norms) or their modification, but have some
control on how they can be used. An average citizen has no power to change the
Law, and has no power on what norms are effective in the jurisdiction she is
situated in. These powers instead are reserved to persons, entities and
institutions specifically designated to do so, for example the parliament and,
under some given constraints, also by judges (in Common Law juridical system,
especially).

However, a citizen can argue that one norm instead of another applies in a
specific case. This amounts to saying that one norm is to be preferred to the
other in the case.

\emph{Prima-facie} conflicts appear in legal systems for a few main reasons,
among which we can easily identify three major representatives: (1) norms from
different sources, (2) norms emitted at different times, and (3) exceptions.
These phenomena are well understood and have given rise to principles which
existed for a long time in legal theory and been used to solve such issues.
These principles are still used in many situations, such as an argument to
drive constitutional judgement against a given norm or a given sentence. Here
we list the three major legal principles, expressing preferences among rules
to be applied \cite{sartor:05}.
\begin{description}
\item[Lex Superior] This principle states that when there is a conflict
between two norms from different sources, the norms originating from the
legislative source higher in the legislative source hierarchy takes precedence
over the other norm. This means that if there is a conflict between a
federal law and a state law, the federal law prevails over the state law. %
 \item[Lex Posterior] According to this principle, a norm emitted after
another norm takes precedence over the older norm.

\item[Lex Specialis] This principle states that when a norm is limited to a
specific set of admissible circumstances, and under more general conditions
another norm applies, the most specific norm prevails.

\end{description}

Besides the above principles a legislator can explicitly establish that one
norm prevails over a conflicting norm.

The intuition behind these principles (and eventually others) is that when
there are two conflicting norms, and the two norms are applicable in a
specific case, we can apply one of these principles to create an instance of a
superiority relation that discriminates between the two conflicting norms.
However, there are further complications. What if several principles apply and
these produce opposite preferences? Do the preferences lead to opposite
outcomes of a case? These are examples of situations when revision of
preferences is relevant. The following example illustrates this situation.



%

Charlie is an immigrant son of an italian, and living in Italy, who is
interested in joining the Italian Army, based on Law 91 of 1992. However, his
application is rejected, based upon a constitutional norm (Article 51 of the
Italian Constitution). The two norms Law 91 and Article 51 are in conflict,
thus the Army's decision is based on the \emph{lex superior} principle.
Charlie appeals against the decision in court. The facts of the case are
undisputed, and so are the norms to be applied and their interpretation. Thus
the only chance for Bob, Charlie's lawyer, to overturn the decision is to
argue that Law 91 overrides Article 51 of the Constitution. Thus Bob
counter-argues appealing to the \emph{lex specialis} principle since Law 91 of
1992 explicitly covers the case of a foreigner who applies for joining the
Army for the purpose of obtaining citizenship.

The two arguments do not discuss about facts and rules that hold in the case.
They disagree about which rule prevails over the other, Article 51 of the
Constitution or Law 91. In particular, Bob's argument can be seen as an
argument where the relative strength of the two rules is reversed compared to
the argument of the Army's lawyer, and it can be used to revise the
previous decision.

While the mechanism sketched above concerns the notion of strategic reasoning,
where a discussant looks at the best argument to be used in a case to prove a
given claim, in this case, that one rule prevails over another rule. However,
the key aspect is that before embarking in this kind of arguments, one has to
ensure that changing a preference leads to a different outcome of the claim of
the case. It is not our aim to study how to justify preferences using
argumentation. In this work, we investigate if it is possible to modify the
extension of a theory (as represented by a defeasible theory) only through
changes on the superiority (preference) relation. Thus, we believe that our
framework is foundational for argumentation of preferences. This means that
one can first determine whether the outcome of a discussion can be turned in
her favour only changing the superiority relation, and then to figure out
which argument (if any) supports the preference.

\medskip

In the current literature about formalisms apt to model normative and legal
reasoning, a simple and efficient non-monotonic formalism which has been
discussed in the community is \textit{Defeasible Logic}. This system is
described in detail in the next sections.

One of the strong aspects of Defeasible Logic is its characterisation in terms
of argumentation semantics \cite{jlc:argumentation}. In other words, it is
possible to relate it to general reasoning structure in non-monotonic
reasoning which is based on the notion of admissible reasoning chain. An
admissible reasoning chain is an argument in favour of a thesis. For these
reasons, much research effort has been spent upon Defeasible Logic, and once
formulated in a complete way it encompasses other (sceptical) formalisms
proposed for legal reasoning
\cite{jlc:argumentation,DBLP:journals/jlp/AntoniouMB00,icail2011carneades}.

Most interestingly, in Defeasible Logic we can reach positive conclusions as
well as negative conclusions, thus it gives understanding to both accept a
conclusion as well as reject a conclusion. This is particularly advantageous
when trying to address the issues determined by reasoning conflicts.

It has been pointed out that the AGM framework for belief revision is not
always appropriate for legal reasoning \cite{GovRot:igpl09}. Moreover, it is
not clear how to apply AGM to preference revision. Accordingly, this paper
provides a comprehensive study of the conditions under which it is possible to
revise a defeasible theory by changing the superiority relation of the theory,

%% file: DefLogPRDL.tex
\section{Defeasible Logic}\label{sec:DefLogPRDL}

A defeasible theory consists of five different kinds of knowledge: facts,
strict rules, defeasible rules, defeaters, and a superiority relation
\cite{tocl}. Examples of facts and rules below are standard in the literature
of the field.

\emph{Facts} denote simple pieces of information that are considered always to
be true. For example, a fact is that ``Sylvester is a cat'', formally
$cat(Sylvester)$. A \emph{rule} $r$ consists of its \emph{antecedent} $A(r)$
which is a finite set of literals, an \emph{arrow}, and its \emph{consequent}
(or \emph{head}) $C(r)$, which is a single literal. A \emph{strict rule} is a
rule in which whenever the premises are indisputable (e.g., facts), then so is
the conclusion. For example,
\[
cat(X) \rightarrow mammal(X)
\] 
means that ``every cat is a mammal''. A \emph{defeasible rule} is a rule
that can be defeated by contrary evidence; for example, ``cats typically eat
birds'':
\[
cat(X) \To eatBirds(X).
\]

The underlying idea is that if we know that something is a cat, then
we may conclude that it eats birds, unless there is evidence proving otherwise. \emph{Defeaters} are rules that cannot be used to draw any conclusion.
Their only use is to prevent some conclusions, i.e., to defeat defeasible
rules by producing evidence to the contrary. An example is ``if a cat has just
fed itself, then it might not eat birds'':
\[
justFed(X) \defeater \neg eatBirds(X).
\]

The \emph{superiority relation} among rules is used to define where one rule
may override the conclusion of another one, e.g., given the defeasible rules
\begin{eqnarray}
	r: cat(X) &\To& eatBirds(X)\nonumber\\
	r^\prime: domesticCat(X) &\To& \neg eatBirds(X)\nonumber
\end{eqnarray}
which would contradict one another if Sylvester is both a cat and a
domestic cat, they do not in fact contradict if we state that $r^\prime$ wins against $r$,
leading Sylvester not to eat birds. Notice that in Defeasible Logic the
superiority relation determines the relative strength of two conflicting
rules, i.e., rules whose heads are \emph{complementary}. The complementary of a
literal $q$ is denoted by $\non q$; if $q$ is a positive literal $p$, then
$\non q$ is $\neg p$, and if $q$ is a negative literal $\neg p$ then $\non q$
is $p$.

Like in \cite{tocl}, we consider only a propositional version of this logic,
and we do not take into account function symbols. Every expression with
variables represents the finite set of its variable-free instances.

A \emph{defeasible theory D} is a triple ($F, R, >$), where $F$ is a finite
consistent set of literals called \emph{facts}, $R$ is a finite set of rules,
and $>$ is an acyclic superiority relation on $R$; given two rules $r$ and
$s$, we will use the infix notation $r>s$ to mean that $(r,s)\in >$. The set
of all strict rules in $R$ is denoted by $R_s$, and the set of strict and
defeasible rules by $R_{sd}$. We name $R[q]$ the rule set in $R$ with head
$q$. A \emph{conclusion} of $D$ is a tagged literal and can have one of the
following forms:
\begin{enumerate}[1.]
\item $+\Delta q$, which means that $q$ is definitely provable in $D$, i.e.,
there is a definite proof for $q$, that is a proof using facts, and strict
rules only;
\item $-\Delta q$, which means that $q$ is definitely not provable in  $D$ (i.e., a definite proof for $q$ does not exist);
\item $+\partial q$, which means that $q$ is defeasibly provable in $D$;
\item $-\partial q$, which means that $q$ is not defeasibly provable, or refuted in $D$.
\end{enumerate}

A \emph{proof} (or \emph{derivation}) is a finite sequence
\mbox{$P=(P(1), \dots, P(i))$} of tagged literals where for each $n$, $0\leq
n\leq i$ the following conditions (proof conditions) are
satisfied and $P(1..i)$ denotes the initial part of the sequence of
length $i$.

\begin{tabbing}
 $+\Delta$: \= If $P(n+1)=+\Delta q$ then\+\\
 (1) $q\in F$ or\\
 (2) $\exists r\in R_s[q]\forall a\in A(r): +\Delta a\in P(1..n)$.
\end{tabbing}

The negative proof conditions for $\Delta$ are the \emph{strong negation} of
the positive counterpart: this is closely related to the function that
simplifies a formula by moving all negations to an inner most position in the
resulting formula, and replaces the positive tags with the respective negative
tags, and the other way around \cite{ecai2000-5,igpl09policy}.

\begin{tabbing}
 $-\Delta$: \= If $P(n+1)=-\Delta q$ then\+\\
 (1) $q\notin F$ and\\
 (2) $\forall r\in R_s[q]\exists a\in A(r): -\Delta a\in P(1..n)$.
\end{tabbing}

The proof conditions just given are meant to represent forward
chaining of facts and strict rules ($+\Delta$), and that it is not
possible to obtain a conclusion just by using forward chaining of
facts and strict rules ($-\Delta$).

The proof conditions for $\pm\partial$ are as follows:

\begin{tabbing}
\=7890\=1234\=5678\=9012\=3456\=\kill
\>$+\partial$: \> If $P(n+1)=+\partial q$ then either \\
\>\>(1) $+\Delta q \in P(1..n)$, or \\
\>\>(2) \> (2.1) $-\Delta\non q \in P(1..n)$ and \\
\>\>\>(2.2) $\exists r\in R_{sd}[q]$ such that $\forall a \in
A(r): +\partial a\in P(1..n)$ and \\
\>\>\>(2.3) $\forall s \in R[\non q]$ either \\
\>\>\>\>(2.3.1) $\exists a\in A(s)$ such that $-\partial a\in P(1..n)$, or \\
\>\>\>\>(2.3.2) \= $\exists t\in R_{sd}[q]$ such that \\
\>\>\>\>\> $\forall a\in A(t): +\partial a\in P(1..n)$ and $t>s$.
\end{tabbing}

\begin{tabbing}
\=7890\=1234\=5678\=9012\=3456\=\kill
\>$-\partial$: \> If $P(n+1)=-\partial q$ then  \\
\>\>(1) $+\Delta q \not\in P(1..n)$ and either \\
\>\>(2) \> (2.1) $+\Delta\non q \in P(1..n)$, or \\
\>\>\>(2.2) $\forall r\in R_{sd}[q] \ \exists a \in
A(r): -\partial a\in P(1..n)$ or \\
\>\>\>(2.3) $\exists s \in R[\non q]$ such that \\
\>\>\>\>(2.3.1) $\forall a\in A(s): +\partial a\in P(1..n)$ and \\
\>\>\>\>(2.3.2) \= $\forall t\in R_{sd}[q]$ either \\
\>\>\>\>\> $\exists a\in A(t): -\partial a\in P(1..n)$, or $t\not> s$.
\end{tabbing}

The main idea of the conditions for a defeasible proof ($+\partial$) is that
there is an applicable rule (i.e., a rule where all the antecedents are
defeasibly proved) and every rule for the opposite conclusion is either
discarded (i.e., one of the antecedents is not defeasibly
provable) or defeated by a stronger applicable rule for the
conclusion we want to prove. The conditions for the negative proof tags (e.g.,
$-\partial$) show that any
systematic attempt to defeasibly prove the conclusion fails. The conditions
for $+\Delta$ and $-\Delta$, and $+\partial$ and $-\partial$ are related by
the \emph{Principle of Strong
Negation }introduced in \cite{ecai2000-5}. The key idea behind this principle
is that conclusions labelled with a negative proof tag are the outcome of a
constructive proof that the corresponding positive conclusion cannot be
obtained (and the other way around). The principle states that the inference
conditions for a pair of proof tags $+\#$ and $-\#$ are the strong negation of
the other, where the strong negation of a condition corresponds essentially to
the function that simplifies a formula by moving all negations to an
innermost position in the resulting formula (for the full details see
\cite{ecai2000-5}).

%

%

As usual, given a proof tag $\#$, a literal $p$ and a theory $D$, we use
$D\vdash\pm\# p$ to denote that there is a proof $P$ in $D$ where for some
line $i$, $P(i)=\pm\# p$. Alternatively, we say that $\pm\#p$ holds in
$D$, or simply $\pm\#p$ holds when the theory is clear from the context.

The set of positive and negative conclusion is called \emph{extension}.
Formally,
\begin{defn}\label{def:extension}
Given a defeasible theory $D$, the \emph{defeasible extension} of $D$ is
defined as
\[
E(D) = (+\partial, -\partial),
\]
where $\pm\partial = \set{l: l \text{ appears in } D \text{ and } D \vdash
\pm\partial l}$.
\end{defn}

Due to the nature of the revision operators discussed in this paper, the
extension does not contain strict conclusions since the only way to modify
them is to operate on the set of strict rules (i.e., addition or removal).
Similarly, the extension will not include information about the proof tags
introduced below. Such proof tags are useful to identify structures in
proofs and \emph{where} to operate in the theory, but they do not specify
\emph{what} is defeasibly provable, or not.

The inference mechanism of Defeasible Logic does not allow us to derive
inconsistencies unless the monotonic part of the logic is inconsistent, as
clarified by the following definition.
\begin{defn} 
	A defeasible theory $D$ is inconsistent iff there exists a literal $p$ such that (($D\vdash +\partial p$ and \mbox{$D\vdash +\partial \non p$}) iff ($D\vdash +\Delta p$ and $D\vdash +\Delta \non p$)).
\end{defn}

In this paper, we do not make use of strict rules, nor defeaters\footnote{The
restriction does not result in any loss of generality: (1) the superiority
relation does not play any role in proving definite conclusions, and (2) for
defeasible conclusions \cite{tocl} proves that it is always possible to remove
(a) strict rules from the superiority relation and (b) defeaters from the
theory to obtain an equivalent theory without defeaters and where the strict
rules are not involved in the superiority relation. A consequence of this
assumption is that the theories we work with in this paper are consistent.},
since every revision changes only the priority among defeasible rules (the
only rules that act in our framework), but we need to introduce eight new
types of auxiliary tagged literals, whose meaning is clarified in
Example~\ref{exConclusion}. As it will be clear in the remainder, they will be
significantly useful in simplifying the categorisation process, and
consequently, the revision calculus.

\begin{example}\label{exConclusion}
Let $D$ be the following defeasible theory:

\begin{align*}
    F = \emptyset \\
    R = \{ & r_{1}:\ \To a & & r_{7}:\ \To b\\
	& r_{2}:\ a \To c & & r_{8}:\ \To \neg c\\
	& r_{3}:\ c \To d & & r_{9}:\ \To \neg b\\
	& r_{4}:\ \To \neg a & & r_{10}:\ \To e\\
	& r_{5}:\ \To \neg d & & r_{11}:\ \To f \}\\
	& r_{6}:\ \neg d \To p & & \\
     > =\{ & (r_{1}, r_{4}), (r_{5}, r_{3})\}.\\
\end{align*}
To improve readability, from now on we use the following graphical
notation to represent a theory like the previous one:

\setlength{\arraycolsep}{1pt}
\[
\begin{array}{cccccccc}
\To_{r_1} & a      & \To_{r_2} & c      & \To_{r_3} & d      &           &\\
\vee      &        &           &        & \wedge    &        &           &\\
\To_{r_4} & \neg a &           &        & \To_{r_5} & \neg d & \To_{r_6} & p\\
          &        &           &        &           &        &           &\\
\To_{r_7} & b      & \To_{r_8} & \neg c &           &        &           &\\
          &        &           &        &           &        &           &\\
\To_{r_9} & \neg b &           &        &           &        &           &\\
          &        &           &        &           &        &           &\\
\To_{r_{10}}&      e & \To_{r_{11}}&  f     &           &        &           &
\end{array}
\]
where the $\wedge$ and $\vee$ symbols in the graphical representation of a
theory are not conjunctions and disjunctions but they represent the
superiority relation $>$. In the example, $\vee$ means that $r_1 >r_4$ and $\wedge$ that $r_5 >r_3$.
\end{example}

A conclusion in a defeasible proof can now take one or more of the following forms:

\begin{enumerate}[1.]
\setcounter{enumi}{4}
\item $+\Sigma q$, which means there is a reasoning chain supporting $q$;
\begin{itemize}
	\item $r_{1}$, $r_{2}$, $r_{3}$ form a chain supporting literal $d$ (+$\Sigma d$).
\end{itemize}

\item $-\Sigma q$, which means there is no reasoning chain supporting
$q$;
\begin{itemize}
	\item Since there are no rules for literal $\neg p$, then we have $-\Sigma \neg p$.
\end{itemize}

\item $+\sigma q$, which means there exists a reasoning chain supporting $q$ that is not defeated by any applicable reasoning chain attacking it;
\begin{itemize}
	\item $r_{1}$, $r_{2}$ and $r_{7}$, $r_{8}$ are two undefeated chains for $c$ and $\neg c$, respectively; Thus, we have $+\sigma c$, $+\sigma \neg c$.
\end{itemize}

\item $-\sigma q$, which means that every reasoning
chain supporting $q$ is attacked by an applicable reasoning chain;
\begin{itemize}
	\item Every chain for $d$ is defeated ($-\sigma d$, notice that there exists only one in this case).
\end{itemize}

\item $+\omega q$, which means there exists a reasoning chain supporting
$q$ that defeasibly proves all its antecedents;
\begin{itemize}
	\item In the chain $r_{1}$, $r_{2}$, $r_{3}$, only rule $r_{3}$ is defeated, hence $+\omega d$ holds.
\end{itemize}

\item $-\omega q$, which means that in every reasoning chain supporting $q$,
at least one of its antecedents is not defeasibly provable.
\begin{itemize}
	\item Since $+\partial b$ does not hold, we can conclude $-\omega \neg c$.
\end{itemize}

\item $+\varphi q$, which means there exists a reasoning chain that defeasibly
proves $q$ made of elements such that there does not exist  any rule for the
opposite conclusion;
\begin{itemize}
	\item There are no rules for $\neg e$, thus $+\varphi e$ holds.
\end{itemize}

\item $-\varphi q$, which means that for every reasoning chain supporting $q$ there exists an element such that a rule for the opposite conclusion could fire;
\begin{itemize}
	\item $r_{4}$ supports $\neg a$, hence we have $-\varphi a$.
\end{itemize}
\end{enumerate}

The tagged literals are formally defined by the following proof conditions.
Again, the negative counterparts are obtained by the principle of strong
negation. An important consequence of using this principle to formulate the
conditions for asserting tagged literals is that for any literal $p$ and any
proof tag $\#$, it is not possible to have both $+\#p$ and $-\#p$ (the
interested reader is referred to \cite{ecai2000-5,igpl09policy}). 

Such proof tags identify structures of rules and demonstrations that are
significant for the revision operations when we change the superiority
relation. For example, $+\Sigma p$ means that we could use Modus Ponens (or
\emph{forward chaining}) for deriving $+\partial p$.

\begin{tabbing}
  $+\Sigma$: \= If $P(n+1)=+\Sigma q$ then\+\\
  (1) $+\Delta q \in P(1..n)$ or\\
  (2) $\exists r\in R_{sd}[q] $ such that $\forall a\in A(r): +\Sigma a\in P(1..n)$
\end{tabbing}

\begin{tabbing}
  $-\Sigma$: \= If $P(n+1)=-\Sigma q$ then\+\\
  (1) $+\Delta q \not\in P(1..n)$ and\\
  (2) $\forall r\in R_{sd}[q]: \exists a\in A(r) $ such that $ -\Sigma a\in P(1..n)$
\end{tabbing}
The definitions of $\pm\Sigma$ formalise the concept of \emph{chain} leading
to a given literal. 

With respect of the analysis on how to change a theory by
only acting on the superiority relation, if there does not exist any chain
leading to a literal $p$ (i.e., $-\Sigma p$ holds), then no modification of
the theory is possible to prove $p$.

\begin{tabbing}
\=7890\=1234\=5678\=9012\=3456\=\kill
  $+\sigma$: \= If $P(n+1)=+\sigma q$ then\+\\
  (1) $+\Delta q \in P(1..n)$ or\\
  (2) \=  $\exists r\in R_{sd}[q]$ such that\+\\ 
	  (2.1) $\forall a\in A(r): +\sigma a\in P(1..n)$
      and\\
      (2.2) \= $\forall s\in R[\non q] \exists a\in A(s)$ such that \\ \> $-\partial a\in P(1..n)$ or $s \not >r$.
\end{tabbing}

\begin{tabbing}
\=7890\=1234\=5678\=9012\=3456\=\kill
  $-\sigma$: \= If $P(n+1)=-\sigma q$ then\+\\
  (1) $+\Delta q \not\in P(1..n)$ and\\
  (2) \= $\forall r\in R_{sd}[q]:$\+\\ 
	  (2.1) $\exists a\in A(r) $ such that $ -\sigma a\in P(1..n)$
      or\\
      (2.2) \= $\exists s\in R[\non q]$ such that \\
      \> $\forall a\in A(s): +\partial a\in P(1..n)$ and $ s>r$.
\end{tabbing}
Notice that the definitions given above for $\pm\sigma$ are weak forms of the
notion of support proposed in \cite{aaai2000,ecai2000-5} for the definition of
an ambiguity propagating variant of Defeasible Logic, in the sense that these
definitions are less selective than the ones of \cite{aaai2000}. 

The undefeated chain that allows to state $+\sigma p$ may be a good candidate
for the revision process in order to defeasibly prove $p$.

\begin{tabbing}
  $+\omega$: \= If $P(n+1)=+\omega q$ then\+\\
  (1) $+\Delta q \in P(1..n)$ or\\
  (2) $\exists r\in R_{sd}[q] $ such that $ \forall a\in A(r): +\partial a\in P(1..n)$.
\end{tabbing}

\begin{tabbing}
  $-\omega$: \= If $P(n+1)=-\omega q$ then\+\\
  (1) $+\Delta q \not\in P(1..n)$ and\\
  (2) $\forall r\in R_{sd}[q]: \exists a\in A(r) $ such that $ -\partial a\in P(1..n)$.
\end{tabbing}
The chain that allows to state $+\omega p$ represents a defeasible proof for
$p$ that can only fail on the last derivation step. Thus, possible
modifications can focus on this last step instead of considering
the whole chain.

\begin{tabbing}
  $+\varphi$: \= If $P(n+1)=+\varphi q$ then\+\\
  (1) $+\Delta q \in P(1..n)$ or\\
  (2) \=  $\exists r\in R_{sd}[q]$ such that \+\\
	  (2.1) $\forall a\in A(r): +\varphi a\in P(1..n)$
      and\\
      (2.2) \= $\forall s\in R[\non q]:\exists a\in A(s) $ such that $
        -\Sigma a\in P(1..n)$.
\end{tabbing}

\begin{tabbing}
  $-\varphi$: \= If $P(n+1)=-\varphi q$ then\+\\
  (1) $+\Delta q \not\in P(1..n)$ and\\
  (2) \=  $\forall r\in R_{sd}[q]:$\+\\
	  (2.1) $\exists a\in A(r) $ such that $ -\varphi a\in P(1..n)$ or\\
      (2.2) \= $\exists s\in R[\non q] $ such that $\forall a\in A(s):
        +\Sigma a\in P(1..n)$.
\end{tabbing}

The definition of $+\varphi$ ensures that it is not possible to have
a counter-argument for a reasoning chain, i.e., a proof, for a literal tagged
with it. In particular, we can not have a direct attack, nor an attack to one
of its arguments. Therefore, no modification on the superiority relation is
possible to reject a literal tagged with $+\varphi$.

Given the above definitions, it is straightforward to derive the
implication chains reported below in Figure \subref*{figImplications+} --
\subref{figImplications-} using techniques presented in
\cite{tocl:inclusion}.

\begin{figure}[htp]
\centering \subfloat[Positive implication chain]{
\begin{tikzpicture}[scale=0.8]
  \node (D) at (-4.5,0) {$+\Delta$};
  \node (phi) at (-3,0) {$+\varphi$};
  \node (par) at (-1.5,0) {$+\partial$};
  \node (a) at (0,1) {$+\omega$};
  \node (S) at (0,-1) {$+\sigma$};
  \node (o) at (1.5,0) {$+\Sigma$};
 \begin{scope}[>=stealth,->,thick,shorten >=2pt,shorten <=2pt]
  \draw (D) -- (phi);
  \draw (phi) -- (par);
  \draw (par) -- (a);
  \draw (par) -- (S);
  \draw (a) -- (o);
  \draw (S) -- (o);
 \end{scope}
\end{tikzpicture}
\label{figImplications+}}
\quad \medskip \subfloat[Negative implication chain]{
\begin{tikzpicture}[scale=0.8]
  \node (D) at (4.5,0) {$-\Delta$};
  \node (phi) at (3,0) {$-\varphi$};
  \node (par) at (1.5,0) {$-\partial$};
  \node (a) at (0,1) {$-\omega$};
  \node (S) at (0,-1) {$-\sigma$};
  \node (o) at (-1.5,0) {$-\Sigma$};
 \begin{scope}[>=stealth,->,thick,shorten >=2pt,shorten <=2pt]
  \draw (o) -- (a);
  \draw (o) -- (S);
  \draw (a) -- (par);
  \draw (S) -- (par);
  \draw (par) -- (phi);
  \draw (phi) -- (D);
 \end{scope}
\end{tikzpicture}\label{figImplications-}}
\caption{Implication chains.}\label{figImplications}
\end{figure}
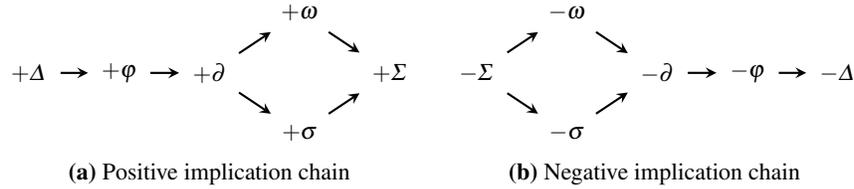

One could be tempted to say that $+\sigma$ implies $+\omega$ (and
symmetrically, $-\omega$ implies $-\sigma$). This is not the case. Indeed, if
we consider theory $D$ of Example~\ref{exConclusion}, we have: (i.) $+\omega
d$ and $-\sigma d$, (ii.) $+\sigma \neg c$ and $-\omega \neg c$.

To better explain how the new proof tags behave, we report in
Table~\ref{tab:Conclusion} the set of all conclusions. For each literal, we
only report the proof tag which is minimal with respect to the orderings given
in Figure~\ref{figImplications}. For example, $+\partial a$ means that we
prove $+\omega a,\, +\sigma a,\, +\Sigma a$, but $-\varphi a$. If no tag is reported, then it is not possible to derive the literal with any tags with respect to the ordering given in Figure~\ref{figImplications}.

\begin{table}
\begin{center}
\begin{tabular}{cccccccc}
\toprule[.5mm]
    & $a$  & $b$ & $c$  & $d$  & $e$  & $f$  & $p$\\
 \midrule
 $+$ & $+\partial$ & $+\sigma$ & $+\partial$ & $+\omega$ & $+\varphi$ & $+\varphi$ & $+\partial$\\
 $-$ & $-\varphi$ & $-\partial$ & $-\partial$ & $-\sigma$ & & & $-\varphi$\\
 \midrule[.5mm]  
     & $\neg a$ & $\neg b$ & $\neg c$ & $\neg d$& $\neg e$ & $\neg f$ & $\neg p$\\
 \midrule
 $+$ & $+\omega$ & $+\sigma$ & $+\sigma$ & $+\partial$ &  & & \\
 $-$ & $-\partial$ & $-\partial$ & $-\omega$ & $-\varphi$ & $-\Sigma$ & $-\Sigma$ & $-\Sigma$\\
 \bottomrule[.5mm]
\end{tabular}
\end{center}\caption{Conclusions for literals in Example~\ref{exConclusion}.}\label{tab:Conclusion}
\end{table}



We report now some theoretical results: these are useful during the revision
process described in Section \ref{sec:PrefDefRevisionPRDL}.
Proposition~\ref{propFi_p_Implies-Sigma_Np} highlights the fact that, given a
theory $D$ and literal $q$, if $D\vdash +\varphi q$ then there are no chains
for the complementary. Notice that, in general, the opposite does not hold (as
for literal $p$ in Example \ref{exConclusion}).

\begin{prop}\label{propFi_p_Implies-Sigma_Np}
Given a consistent defeasible theory $D$, if $D \vdash +\varphi p$
for a literal $p$ (with $p \not\in F$), then $D \vdash -\Sigma\non p$.
\end{prop}
\begin{proof}
The proof straightly follows from the definition of $+\varphi$, condition (2.2). This condition must hold for each element in the chain, as well as for $p$.
\end{proof}

The next proposition formally states the following idea: given a defeasibly
proved literal $p$ and a chain leading to $\non p$ with all the antecedents
defeasibly proved, then such a chain is defeated by a priority rule at the
last proof step (by the rule proving $p$).

\begin{prop}\label{propPartial_p_AND_Omega_Np_implies_-sigma_Np}
Given a consistent defeasible theory $D$, if $D\vdash +\partial p$ and 
$D\vdash +\omega\non p$ for a literal $p$ (with $p \not\in F$), then $D\vdash -\sigma\non p$.
\end{prop}
\begin{proof}
By definition of $+\partial$, we have that the condition below
\begin{tabbing}
\=7890\=1234\=5678\=9012\=3456\=\kill
(2.3) $\forall s \in R[\non p]$ either \\
\>\>(2.3.1) $\exists a\in A(s): -\partial a\in P(1..n)$ or \\
\>\>(2.3.2) $\exists t\in R_{sd}[p]$ such that \\
\>\>\> $\forall a\in A(t): +\partial a\in P(1..n)$ and $t>s$
\end{tabbing}

\noindent holds for $p$. In fact condition (2.3.2) has to be true
since we know condition (2.3.1) is not, because
\[
\left.
\begin{matrix}
D\vdash +\partial p \text{ implies }\exists r \in R[p]. \forall a \in A(r): +\partial a\\
D\vdash +\omega\non p \text{ implies }\exists s \in R[\non p]. \forall a \in
A(s): +\partial a
\end{matrix} \right\} \text{ thus }\]
\[\exists t \in R[p]. \forall a \in A(t): +\partial
a \text{ and } t>s.
\]

\noindent This is condition (2.2) for $-\sigma \non p$. Moreover, since all the
premises of $\non p$ are defeasibly proved by hypothesis and we have proved
that the chain is defeated, then it has to loose on the last proof step.
\end{proof}

We now capture the concept of a derivation based on a contradiction. To do so,
we begin by defining what the meaning of dependency between literals is;
afterwards, we look at how the notion of $\partial$\emph{-unreachability}
defines literals whose derivation is based upon an inconsistency.

\begin{defn}\label{def:dependency}
Let $a$ and $b$ be two literals. Then $a$ \emph{depends on} $b$ iff \emph{(1)} $b = a$ or \emph{(2)} $\,\forall r \in R[a]$, either \emph{(2.1)} $b \in
A(r)$, or \emph{(2.2)} $\exists c$ such that $c \in A(r)$, and $c$ depends on $b$.
\end{defn}

\noindent The following result shows that a defeasibly proved literal also implies the provability of all literals it depends on. In other words, the
property of dependency given above propagates backwards the defeasible
provability of literals.

\begin{prop}\label{prop:partialp_depends_on_q_partialq} 
Given a defeasible theory $D$, if $D \vdash +\partial p$ and $p$ depends on
$q$, then \mbox{$D \vdash +\partial q$}.
\end{prop} 

\begin{proof}
The proof is by case inspection of Definition~\ref{def:dependency}. If clause
(1) holds, the claim trivially follows. For the other cases, the proof is by
induction on the degree of dependency between literals. A literal $a$ depends
on $b$ \emph{with degree} 1 if $a$ depends on $b$ and there exists a rule $r$,
with $C(r) = a$ and $b \in A(r)$. A literal $a$ depends on $b$ \emph{with
degree} $n+1$ if $a$ depends on $b$ and there is a literal $c$ such that $a$
depends on $c$ with degree 1 and $c$ depends on $b$ with degree $n$.

For the inductive base (i.e., $p$ depends on $q$ with degree 1), $+\partial p$
means that there is a rule for $p$ with every antecedent defeasibly proved.
Thus, $D \vdash +\partial q$.

For the inductive step, suppose that the property holds up to degree $n$ and
$p$ depends on $q$ with degree $n+1$. By definition, there exists a literal
$c$ such that $p$ depends on $c$ with degree 1, thus $D \vdash +\partial c$
(given $D\vdash +\partial p$ by hypothesis) and $c$ depends on $q$ with degree
$n$. Thus, by inductive hypothesis, $D \vdash +\partial q$. \end{proof}

The next definition identifies literals only
depending on contradictions. For example, consider the theory with the
following rule:
\[
a, \neg a, b\To_{r} p.
\]
For deriving $+\partial p$ we need both $+\partial a$ and $+\partial \neg a$,
and this is possible only in the case that the theory is inconsistent.
However, we have also to cater for situations where the dependency is not
direct, for example in theories like
\[
a \To_{r_{1}} b \qquad \neg a,b\To_{r_{2}} p.
\] 
\begin{defn}\label{def:unreach}
A literal $p$ is \emph{$\partial$-unreachable} iff $\,\forall r
\in R[p],$ either \emph{(1)} $\exists l,\exists a,b \in A(r)$ such that
\emph{(1.1)} $a$ depends on $l$, and \emph{(1.2)} $b$ depends on $\non l$, or
\emph{(2)} $\exists d \in A(r)$ such that $d$ is $\partial$-unreachable. Otherwise, we define $p$ to be $\partial$-reachable.
\end{defn}

\noindent The result below formalises the relationship between
$\partial$-unreachable literals and inconsistent theories.

\begin{prop}\label{prop:inconsitence-unreachable}
Given a theory $D$, let $p$ be a $\partial$-unreachable literal. If $D\vdash
+\partial p$, then $D$ is inconsistent.
\end{prop}
\begin{proof}
The proof is by induction on the number of $\partial$-unreachable literals in
a derivation.

For the base case, $p$ is the only $\partial$-unreachable literal in its
derivation. Given that $D \vdash +\partial p$, there is a rule for $p$ such
that all its antecedents are provable. By Definition~\ref{def:unreach}, for
every rule for $p$ there are two antecedents $a$ and $b$ depending on a
literal $l$ and its complement, respectively. Thus, we have both $+\partial a$
and $+\partial b$. By Proposition~\ref{prop:partialp_depends_on_q_partialq},
we have $D\vdash +\partial l$ and $D\vdash +\partial \non l$, then $D$ is
inconsistent.

For the inductive step, we assume that the property holds up to $n$
$\partial$-unreachable literals, and $p$ is the $(n+1)^{th}$
$\partial$-unreachable literal. Beside the case we examined in the inductive
base, we have to consider that the antecedent of a rule contains a
$\partial$-unreachable literal $d$ and $D\vdash +\partial d$. Thus, $d$ falls
under the inductive hypothesis, therefore $D$ is inconsistent.
\end{proof}

\noindent The following proposition states that if there is a
$\partial$-reachable literal $p$ with at least one supporting chain, then it
is always possible to defeasibly prove $p$. In other words, the problem of
modifying the superiority relation to pass from $-\partial p$ to $+\partial p$
(or from $+\partial \non p$ to $+\partial p$) has always a solution, provided
that there exists a non-contradictory support.

\begin{prop}\label{prop:+sigma+partial}
Given a consistent defeasible theory $D = (F, R, >)$ and a
$\partial$-reachable literal $p$ with $D \vdash +\Sigma p$, there exists
a theory $D' = (F, R, >')$ such that $D' \vdash +\partial p$.
\end{prop}

\begin{proof}

Proposition~\ref{prop:inconsitence-unreachable} shows that a
$\partial$-unreachable literal is provable only when the theory is
inconsistent, which is against the hypothesis of the proposition.

Suppose that $D \vdash +\Sigma p$ for a theory $D$. Then, there is at least
one reasoning chain $C$ supporting $p$. Among all the possible superiority
relations based on $F$ and $R$, there is a superiority relation $>'$
where every rule $r:\ A \To c$ in $C$ is superior to any rule for $\non c$. Thus, theory $D^{\prime} = (F, R,
>')$ is such that $D' \vdash +\partial p$. \end{proof}


To illustrate why both conditions of Proposition~\ref{prop:+sigma+partial}
are required to guarantee that the canonical case whose outcome is $+\partial
p$ after the revision operation, let us consider a theory with the following
rules:
\setlength{\arraycolsep}{1pt}
\[
\begin{array}{ccccc}
\To_{r_1}  &  a      &            &             &     \\
\To_{r_2}  & \neg a  &            &             &     \\
           &         & a, \neg a  &  \To_{r_3}  &  p.  \\     
\end{array}
\]
In this case we have both $+\Sigma a$ and $+\Sigma\neg a$, therefore we can build a
reasoning chain to $p$, but $p$ itself is $\partial$-unreachable because
it depends on a contradiction. Thus, there is no way to change the previous theory to prove $p$.

\medskip 

\noindent We end this section proposing an example to translate a real-life
case into our logic. This will also help in giving an intuitive revision
mechanism that shows how argumentation in legal reasoning is easily mapped in
changing the superiority relation of a defeasible theory. 

\begin{example}\label{ex:Preferences}

A couple can have offspring but, since both mother and father are affected of
cystic fibrosis, they know that every their child will be most likely affected
by the same genetic anomaly. Since they want their offspring to be healthy,
they request for medically assisted reproduction techniques. Their case is
disputed first in an Italian Court where the judge has to establish which
between Art. 4 of Italian Legislative Act 40/2004\footnote{Art. 4 of Italian
Legislative Act 40/2004: ``\emph{The recourse to medically assisted
reproduction techniques is allowed only [\dots] in the cases of sterility}''.}
($r_{0}$ and $r_{1}$) and the \emph{standard common medical practice}
($r_{3}$) in force in 15 countries of the EU prevails.

The couple is indeed able to produce embryos and they cannot be considered as
sterile ($r_{2}$). This makes both Art. 4 and the \emph{standard common medical practice} to be applicable to their case. The judge
argues in favour of $r_{1}$ based on \emph{lex superior} and refuses their
request: this
principle applies since Art. 4 Act 40/2004 is a legal rule while $r_{3}$ has a
juridical validity but it is not a proper legal rule.

\begin{align*}
	F =\set{& \mathit{Embryo},\ \mathit{GeneticAnomalies}}\\
	R =\{ & r_{0}: \neg \mathit{CandidateInVitroFertilization} \To \neg \mathit{Techniques},\\
	& r_{1}: \neg \mathit{Sterility} \To \neg\mathit{CandidateInVitroFertilization},\\
	& r_{2}: \mathit{Embryo} \To \neg \mathit{Sterility},\\
	& r_{3}: \neg \mathit{Sterility}, \mathit{GeneticAnomalies} \To \mathit{CandidateInVitroFertilization},\\
	& r_{4}: \neg \mathit{Sterility}, \mathit{GeneticAnomalies} \To \neg \mathit{Healthy}\\
	& r_{5}: \mathit{GeneticAnomalies}, \mathit{CandidateInVitroFertilization} \To \mathit{Healthy}\}\\
	> =\{ & (r_{1}, r_{3})\}.		
\end{align*}
The couple appeals to the European Court for Human Rights. The Court
establishes that not permitting the medical techniques would demote the goal
of family health promoted by Article 8 of the Convention. In our example,
$r_{3}$ promotes the goal of family health ($r_{5}$), and thus we invert the
priority between $r_{1}$ and $r_{3}$ based both on \emph{lex superior} and
\emph{lex specialis} ($>' =\{ (r_{3}, r_{1})\}$). In here, the \emph{lex
superior} principle applies because $r_{3}$ is an european directive, while
the \emph{lex specialis} principle applies because $r_{3}$ covers a more
specific case than $r_{0}$. 
\end{example}

%% file: PrefDefRevisionPRDL.tex
\section{Changing defeasible preferences}\label{sec:PrefDefRevisionPRDL}

We now analyse the processes of revision in a defeasible theory, when no
changes to rules and facts are allowed. Henceforth, when no confusion
arises, every time we speak about a (revision) transformation we refer to a
(revision) transformation acting only on the superiority relation.

A good starting point for our investigation is to focus on the corresponding
decision problem (i.e., answering the question \emph{if} it is possible to
modify the state of a literal in a defeasible theory by changing the relative
strength of rules) and characterise it in a formal way (so as to be able also
to answer the question \emph{when}). In particular, in
Subsection~\ref{subsec:np_completeness} we show that the decision problem at
hand is computationally hard in general, while in the remaining of the section
we partition the decision problem in three sub-cases which correspond to the
possible ways in which we can modify the (provability) state of a literal.

\subsection{NP-Completeness} 
\label{subsec:np_completeness}


First, we introduce some additional terminology. Definition~\ref{def:basedOne}
constructs a defeasible theory starting from a fixed set of rules and an empty
set of facts. This formulation limits the revision problem to preference
changes, notwithstanding the particular instance of the superiority relation.

\begin{defn}\label{def:basedOne}
	Given a set of defeasible rules $R$, a defeasible theory $D$ is \emph{based on} $R$ iff \[D = (\emptyset, R, >).\]
\end{defn}

The aim of Definition~\ref{def:refutabilityLiteral} is to specify
the possible relationships between a literal and all theories based on a set
of rules $R$.

\begin{defn}\label{def:refutabilityLiteral}
	Given a set of defeasible rules $R$, a literal $p$ is
\begin{enumerate}
	\item $>$-$R$-\emph{tautological} iff for all theories $D$ based on $R$, $D \vdash +\partial p$.
	\item $>$-$R$-\emph{non-tautological} iff there exists a theory $D$ based on $R$ such that $D \not\vdash +\partial p$.
	\item $>$-$R$-\emph{refutable} iff there exists a theory $D$ based on $R$ such that $D \vdash -\partial p$.
	\item $>$-$R$-\emph{irrefutable} iff for all theories $D$ based on $R$, $D \not\vdash -\partial p$.	
\end{enumerate}	
\end{defn}

The notion of $>$-$R$-\emph{irrefutable} represents the negative counterpart
of $>$-$R$-\emph{tautological}; the same holds for $>$-$R$-\emph{refutable}
and $>$-$R$-\emph{non-tautological}.

If $p$ is $>$-$R$-\emph{tautological}, then, in every theory based on the set
of rules $R$, an instance of the superiority relation such that $p$ is
defeasibly refuted does not exist. Accordingly, if a literal is $>$-$R$-\emph{tautological}, then we cannot revise it.

On the contrary, if an instance of the
superiority relation where $p$ is no longer provable exists, then $p$ is
$>$-$R$-\emph{refutable}.

\medskip

\noindent To prove the NP-completeness of the problem of establishing
if it is possible to revise a theory modifying only the superiority relation,
we reduce the \emph{$3$-SAT problem} -- a known NP-complete problem -- to our
decision problem. In particular, we are going to map a 3-SAT formula to a
defeasible theory and we check wether the literal corresponding to the 3-SAT
formula is tautological. Definition~\ref{def:tranformation} exhibits the
reduction adopted.


\begin{defn}\label{def:tranformation}
	Given a 3-SAT formula $\Gamma = \bigwedge\limits_{i=1}^{n} C_{i}$ such that $C_{i} = \bigvee\limits_{j =1}^{3} a_{j}^{i}$, we define the $\Gamma$-transformation as the operation that maps $\Gamma$ into the following set of defeasible rules
	\begin{align*}
    R_{\Gamma} =\{ 
  	& r_{i j}^{a}: \ \To a_{j}^{i}\\
 	& r_{i j}: a_{j}^{i} \To c_{i}\\
	& r_{\non i}: \To \non c_{i}\\
	& r_{i}: \non c_{i} \To p\}.\\
\end{align*}
\end{defn}

\noindent The above definition clearly shows that the mapping is
polynomial in the number of literals appearing in the 3-SAT formula $\Gamma$.

The second step of the proof construction is to ensure that the proposed
mapping always allows the revision problem to give a correct answer (either
positive, or negative) for every 3-SAT formula.
Proposition~\ref{prop:decisiveAcyclic} and Lemma~\ref{lem:decisive} guarantee
that any theory obtained by $\Gamma$-transformation provides an answer.
These results are also intended to establish relationships between
the notions of tautological and refutable given in
Definition~\ref{def:refutabilityLiteral}.

\begin{defn}\label{def:decisive}
A defeasible theory $D$ is \emph{decisive} iff for every literal $p$ in $D$
either $D \vdash +\partial p$, or $D\vdash -\partial p$.
\end{defn}

\begin{prop}\label{prop:decisiveAcyclic}
Given a defeasible theory $D$, if the atom dependency graph of $D$ is acyclic, then $D$ is decisive.
\end{prop}

\begin{proof}
For a detailed definition of atom dependency graph and a
complete proof of the claim, the interested reader should refer to
\cite{AntoniouBGM06}.
\end{proof}

\begin{lem}\label{lem:decisive}
	Any defeasible theory $D$ based on $R_{\Gamma}$ of Definition~\ref{def:tranformation} (for any $\Gamma$) is decisive.
\end{lem} 

\begin{proof}

It is easy to verify by case inspection that the atom dependency graph is
acyclic.
\end{proof}

\noindent We are now ready to introduce the main result of
NP-completeness. First of all, we have to prove that the revision
problem is in NP. Second, we show that it is NP-hard by
exploiting the formulation of the 3-SAT problem and the transformation
proposed in Definition~\ref{def:tranformation}.

\begin{thm}\label{thm:NP-Completeness}
	The problem of determining the revision of a defeasible literal by changing the preference relation is \emph{NP-}complete. 
\end{thm}

\begin{proof}

The proof that the problem is in NP is straightforward. Given a defeasible
theory $D=(F,R,>)$ and a literal $p$ to be revised, an oracle guesses a
revision (in terms of a new preference relation $>'$ applied to $D$) and
checks if the state of $p$ has changed based on the extensions of $D$ and
$D'=(F,R,>')$. The complexity of this check is bound to the computation of the
extensions of $D$ and $D'$, which \cite{complexity} proves to be linear in the
order of the theory.

To prove the NP-hardness, given a 3-SAT formula
$\Gamma = \bigwedge\limits_{i=1}^{n} C_{i}$ such that $C_{i} =
\bigvee\limits_{j =1}^{3} a_{j}^{i}$, a defeasible theory $D$ based on the set
of defeasible rules $R_{\Gamma}$ obtained by $\Gamma$-transformation, and a
literal $p$ in $D$, we show that:

\begin{enumerate}[(1)]
	\item if $p$ is $>$-$R_{\Gamma}$-tautological, then $\Gamma$ is not satisfiable; 
	\item if $p$ is $>$-$R_{\Gamma}$-non-tautological, then $\Gamma$ is satisfiable.
\end{enumerate}

\noindent (1) Lemma~\ref{lem:decisive} allows us to reformulate the contrapositive using $>$-$R_{\Gamma}$-refutable. If $\Gamma$ is satisfiable, then there exists an interpretation $I$ such that 
\begin{eqnarray}
I \vDash \Gamma &\iff& I \vDash \bigwedge\limits_{i=1}^{n} C_{i} \nonumber\\
  &\iff& I \vDash C_{1} \text{ and } \dots \text{ and } I \vDash C_{n} \nonumber\\
  &\iff&  I \vDash \bigvee\limits_{j =1}^{3} a_{j}^{1}  \text{ and } \dots \text{ and } I \vDash \bigvee\limits_{j =1}^{3} a_{j}^{n}. \nonumber
\end{eqnarray}
Thus, for each $i$, there exists $j$ such that $I \vDash a_{j}^{i}$.

We build a defeasible theory $D'=(\emptyset,R_{\Gamma},>')$ as follows. If
there exists a literal $b_{k}^{l}$ such that $b_{k}^{l} = \non a_{j}^{i}$,
then $(r_{i j}^{a}, r_{l k}^{b})$ is in $>'$. It follows that, by
construction, $D'$ proves $+\partial a_{j}^{i}$. This means that every rule
$r_{i j}$ is applicable and it is not weaker than the corresponding rule
$r_{\non i}$. Hence, we have $-\partial \non c_{i}$, for all $i$.
Consequently, each rule $r_{i}$ for $p$ is discarded and we conclude
$-\partial p$. Accordingly, $p$ is $>$-$R_{\Gamma}$-refutable.

\medskip

\noindent (2) Again, by Lemma \ref{lem:decisive}, every theory based on
$R_{\Gamma}$ is decisive. Thus, $p$ is $>$-$R_{\Gamma}$-refutable.
Accordingly, there exists a theory $D'=(\emptyset, R_{\Gamma}, >')$ such that
$D' \vdash -\partial p$. Given that $R_{\Gamma}[p] \neq \emptyset$ and $R_{\Gamma}[\non p] = \emptyset$ by construction, every rule
for $p$ is discarded. Namely, we have $-\partial \non c_{i}$, for all $i$.

Each rule $r_{\non i}$ is vacuously applicable. Hence, in order to have
$-\partial \non c_{i}$, there must exist a rule $r_{i j}$ that is
applicable. Therefore, for each $i$ there is at least one $j$ such
that $+\partial a_{j}^{i}$.

We build an interpretation $I$ as follows:\footnote{We use the standard notation where $I(a) = 1$ iff $a$ is evaluated to \emph{True} in $I$, and $I(a) = 0$ otherwise.}
\begin{center}
	$I(a_{j}^{i}) = 1$ iff $D
\vdash +\partial a_{j}^{i}$.
\end{center}
Since for each $1\leq i \leq n$, we have $I(a_{j}^{i}) = 1$ for at least one
$j$, then also $I \vDash C_{i}$ for all $i$, and we conclude that $I
\vDash \Gamma$.
\end{proof}
In addition, Theorem~\ref{thm:NP-Completeness} specifies that there
are situations where it is not possible to revise a literal only using the
superiority relation. For example, if we take a 3-SAT formula which is a
tautology, the $\Gamma$-transformation generates a theory that cannot be
revised only using the superiority relation. Thus, we can formulate the following result.
\begin{cor}\label{cor:no-revision}
  There are theories and literals for which a revision modifying only the superiority relation is not possible.
\end{cor}

\subsection{Revision in Legal Domain}

Similarly to what we did in Section \ref{sec:Norms}, we now motivate the types
of changes we are going to study by appealing again to the legal domain. When
two lawyers dispute a case, there are four situations in which each of them
can be if she revises the superiority relation employed by the other one.

\begin{enumerate}[(a)]
\item \label{uno} The revision process supports the argument of
\emph{reasonable doubt}. Someone proves that the rules imply a given
conclusion. If the preference is revised then we can derive that this is not
the case, showing that the conclusion was not beyond reasonable doubt.
\item \label{due} The revision process beats the argument of \emph{beyond
reasonable doubt}. Analogously to situation (\ref{uno}), someone proves that
the rules do
not imply a given conclusion. If the preference is revised, then we can derive
that this is indeed the case.
\item \label{tre} The revision process supports the argument of \emph{proof of
innocence/guilt}. Someone proves that the rules imply a given conclusion. If
the preference is revised, then we can derive that the opposite holds.
\item \label{quattro} The revision process cannot support a given thesis.

\end{enumerate}

\noindent Revising a defeasible theory by changing only the priority among its
rules means studying how an hypothetic revision operator works in the three
cases reported below:
\begin{enumerate}[(1)]
\item \label{can1} how to obtain $-\partial p$, starting from $+\partial p$;
\item \label{can2} how to obtain $+\partial\non p$, starting from $+\partial p$;
\item \label{can3} how to obtain $+\partial p$, starting from $-\partial p$.
\end{enumerate}

\noindent We name these three revisions \textit{canonical}. We provide an
exhaustive analysis, based on the definitions above, in the next subsections.

Situation (\ref{uno}) corresponds to canonical case (\ref{can1}).
Situation (\ref{due}) corresponds to canonical case (\ref{can3}).
Situation (\ref{tre}) corresponds to canonical case (\ref{can2}).

Situation (\ref{quattro}) includes several contexts that are deemed as
sub-cases of the previous ones: in particular, it captures cases where indeed,
the revision process based the superiority relation is impossible, namely:

\begin{itemize}
	\item in the \emph{first} canonical case, when literal $p$ is
$>$-$R$-tautological (by Definition~\ref{def:refutabilityLiteral});
	\item in the \emph{second} canonical case, when literal $p$ is
$>$-$R$-tautological, or a reasoning chain supporting the
complementary does not exist (i.e., condition $-\Sigma \non p$ holds);
	\item in the \emph{third} canonical case, when literal $p$ is
$\partial$-unreachable (as stated in
Proposition~\ref{prop:inconsitence-unreachable}), or a
reasoning chain supporting it does not exist (i.e., condition $-\Sigma p$ holds).
\end{itemize}

Notice that literals provable with tag $+\varphi$ are special cases of
tautological literals (cf. Definition~\ref{def:refutabilityLiteral}).
As such, this kind of literals leads the revision process to be
unsuccessful for the first and the second canonical case. A possible legal scenario is when one of the parties argues in favour of a
thesis in a defeasible way and the counter-part cannot discredit it, or cannot
exhibit a proof for the opposite, independently of the changes in the
superiority relation. The next proposition formally captures the above
intuitions.


\begin{prop}\label{propFi_p_NOWAY_+Partial_Np}
Given a consistent defeasible theory $D=(F,R,>)$, if $D \vdash +\varphi p$
for a literal $p$, then there does not exist a theory $D'=(F,R,>')$ such
that (1) $D'\vdash+\partial\non p$, or (2) $D'\vdash-\partial p$.
\end{prop}

\begin{proof}
(1) Given any theory, to obtain a defeasible proof of a literal $q$,
there must exist at least one reasoning chain supporting $q$, i.e., $+\Sigma
q$ must hold. This is in contradiction with Proposition
\ref{propFi_p_Implies-Sigma_Np} which states that if $+\varphi
\non q$ holds, also $-\Sigma q$ does.

(2) By definition of $+\varphi p$, there exists a reasoning chain that
defeasibly proves $p$ made of elements such that there does not exist any rule
for the opposite conclusion. Thus, no attack to this chain is
possible, and condition (2.3.1) of $+\partial$ always holds for each element
of this chain (we recall that $-\Sigma l$ implies $-\partial l$ for any
literal $l$).
\end{proof}


In the rest of the section we are going to describe three types of revision of
preferences. For each case we identify the conditions under which such
revisions are possible. Therefore, all revision cases studied below will
not consider tautological literals as well as $\partial$-unreachable literals,
assuming that the underlying theories are consistent.

Given a defeasible theory $D$ a literal $p$ can be proved (i.e., $+\partial
p$) or refuted (i.e., $-\partial p$). The three canonical cases cover the
situations where: we pass from a theory proving $p$ to a theory refuting $p$
(without necessarily proving the opposite, $\non p$); we pass from a theory
refuting $p$ to a theory proving $p$; and from a theory proving $p$ to a
theory proving its opposite ($\non p$), and then consequently refuting $p$.
Notice that these three cases are the only ones meaningful involving provable
and refutable literals. In Section~\ref{sec:AGMComparisonPRDL}, we are going
to argue that our canonical cases can be understood as expansion, revision and
contraction of the AGM belief revision framework. Combinatorially, one could
consider another case, where $p$ is refuted and we want to obtain
a theory where we refute $\non p$. However, the meaning of this operation is
not clear to us, and it is partially subsumed by our third canonical case
(given that $+\partial p$ implies
$-\partial\non p$).

We are now ready to go onto the systematic analysis of the combinations
arising from the above defined model. We list the cases by tagging each
macroscopic case by the name \emph{Canonical case} and the combinations
depending upon the analytical schema introduced above by the name
\emph{Instance}. The instances show the combination of proof tags where a
canonical revision is possible, as well as how to operate on the theory to
perform the revision. Where necessary, a general reasoning chain supporting a
literal $p$ will be denoted as $P_{p}$.

\subsection{First canonical case: from $+\partial p$ to $-\partial p$}\label{subsecPartial_p_TO_-Partial_p}

For the sake of clarity, Figure~\ref{figRevTree} gives a tree-based
graphical representation of all analysed instances: for example, the leftmost
leaf (labeled as $+\sigma\non p$) represents the instance where
conditions are $+\sigma\non p$, $+\omega\non p$, $-\varphi p$ and $+\Sigma\non
p$. The scheme will be reprised also in the two remaining canonical cases,
with the appropriate graphical modifications for the particular case.

\begin{figure}[htp]
\centering
\begin{tikzpicture}[scale=.7, font=\small]
  \node[shape=rectangle, text width=3cm, text centered] (root) at (0,0)  {From $+\partial p$ to $-\partial p$ ($-\varphi p \wedge +\Sigma\non p$)};
  \node (+omega) at (-3,-1.5) {$+\omega\non p$};
  \node (+s+o) at (-4.5,-3) {$+\sigma\non p$};
  \node (-s+o) at (-1.5,-3) {$-\sigma\non p$};
  \node (-omega) at (3,-1.5) {$-\omega\non p$};
  \node (+s-o) at (1.5,-3) {$+\sigma\non p$};
  \node (-s-o) at (4.5,-3) {$-\sigma\non p$};
 \begin{scope}[>=stealth,thick,shorten >=2pt,shorten <=2pt]
  \draw (root) -- (+omega);
  \draw (+omega) -- (+s+o);
  \draw (+omega) -- (-s+o);
  \draw (root) -- (-omega);
  \draw (-omega) -- (+s-o);
  \draw (-omega) -- (-s-o);
 \end{scope}
\end{tikzpicture}
\caption{From $+\partial p$ to $-\partial p$: revision cases.}
\label{figRevTree}
\end{figure}
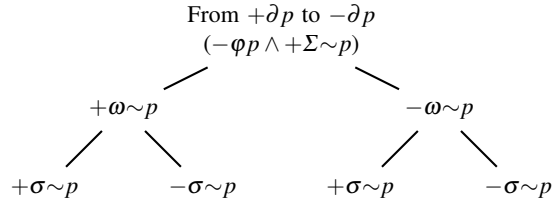

\noindent \textbf{Instance $-\Sigma\non p \wedge +\partial p$:} this case is
not reported in Figure \ref{figRevTree}, but nonetheless it represents a case
worth considering. This means there is no supporting chains
for $\non p$, so we cannot operate on them. Given $-\varphi p$, there
exists at least one of its premises that could be defeated by a rule leading
to the opposite conclusion. Thus, in order to obtain $-\partial p$, we have to
revise the theory allowing at least one of such rules to be able to fire (to
defeat, or at least to have the same power of a rule which actually proves one
of the antecedents in the chain supporting $p$).

\medskip

\noindent \textbf{Instance $+\omega\non p \wedge +\sigma\non p$:} as stated in
Proposition \ref{propPartial_p_AND_Omega_Np_implies_-sigma_Np}, this branch
represents an impossible case for any consistent defeasible theory.

\medskip

\noindent \textbf{Instance $+\omega\non p \wedge -\sigma\non p$:} by the
straightforward implication of Proposition
\ref{propPartial_p_AND_Omega_Np_implies_-sigma_Np}, the chain supporting $\non
p$ fails on the last proof step, defeated by priorities for rules which
defeasibly prove $p$. Thus, we have only to erase these priorities.

\medskip

\noindent \textbf{Instance $-\omega\non p \wedge +\sigma\non p$:} since a chain $P_{\non p}$ exists, and is never defeated (condition $-\omega\non p$
only illustrates that such a chain fails before the last proof step), a
revision process does not have to operate on a chain supporting $p$. We have
to strengthen $P_{\non p}$ by changing many priorities in order to let a rule
in $P_{\non p}$ obtain at least the same strength of such a rule in $P_p$. In
this process, we do not remove any priority among elements in $P_p$, but
only add priorities to let a rule in $P_{\non p}$ win.

\medskip

\noindent \textbf{Instance $-\omega\non p \wedge -\sigma\non p$:} the
reasoning chain $P_{\non p}$ supporting $\non p$ is defeated, but not
necessarily by a chain $P_p$ proving $p$. The case is analogous to the
aforementioned instance, but (1) we probably have to act not only on $P_{\non
p}$, but also on $P_p$, and (2) not only introduce priorities, but erase or invert them. This case represents the most generic
situation, where less information is given: a revision is possible, but we do
not know \emph{a priori} where to change the theory.

\subsection{Second canonical case: from $+\partial p$ to $+\partial\non p$}


By following the cases depicted in Figure \ref{figRevTree}, we explain
how a revision operator should work by changing the root label to ``$+\partial p$ to $+\partial\non p$'' and starting from the same premises
($-\varphi p \wedge +\Sigma\non p$). Once more, our revision tree does not
take into account tags $\pm\varphi$ for the same reasons explained at the
beginning of Section \ref{sec:PrefDefRevisionPRDL}.

\medskip

\noindent \textbf{Instance $+\omega\non p \wedge +\sigma\non p$:}
as stated in Proposition
\ref{propPartial_p_AND_Omega_Np_implies_-sigma_Np} this branch
represents an impossible case for any consistent defeasible
theory.

\medskip

\noindent \textbf{Instance $+\omega\non p \wedge -\sigma\non p$:} Proposition
\ref{propPartial_p_AND_Omega_Np_implies_-sigma_Np} states that the chain
supporting $\non p$ is defeated on the last proof step. This, combined with
$-\sigma\non p$, implies that the last step is defeated by a priority for the
rule which defeasibly proves $p$. In fact, there may exist more than one
defeated chain for $\non p$ on the last step, as well as more than one chain
which proves $p$. We propose two different approaches. We name $P$ the set of
chains defeasibly proving $p$, $P_{ls} \subseteq P$ the set of chains that
defeasibly prove $p$ for which there is a priority that applies at the
last proof step (against a chain that proves $\non p$), and $N$ the set of
chains for which the premises hold:
\begin{enumerate}
\item We choose a chain in $N$. We invert the priority for every chain in $P_{ls}$ that wins at the last proof step. 
We introduce a new priority for making it win against any remaining chain in $P$.
\item In this approach we have two neatly distinguished cases:
\begin{enumerate}
\item $|P_{ls}| > |N|$: for every chain in $N$ we invert the priorities on
the last proof step. For every remaining chain in $P$, we add a priority
between the defeasible rule used in the last proof step of a chain in $N$ and
the rule used in the last proof step of a chain in $P$ (possibly different for
each chain in $N$) such that the chain in $P$ is defeated.
\item $|N| > |P_{ls}|$: we choose a number $|P_{ls}|$ of chains in $N$ and
invert the priority on the step where they are defeated. If at the end of
this step there are still chains in $P$ that defeasibly prove $p$, we go on
with the method used for the case 2(a), focusing on the subset of chains in
$N$ modified during the first step.
\end{enumerate}
\end{enumerate}


These two approaches rely on different underlying ideas. In the first case we
want a unique winning chain. This makes the revision procedure faster than the
second method, since we do not have to choose a different chain to manipulate
every time. Moreover, the number of changes made with the first approach is
equal to that of the second one in the worst case scenario (in general, it
revises the theory with the minimum number of changes).

The strength of the second method relies on the concept of \emph{team
defeater}: we do not give power to a single element, but to a team of rules.
In the first method the entire revision process must be repeated once the
winning chain is defeated, while in the second method if one of the winning
chains is defeated, we have only to apply the revision process on it, and not
on all the other winning chains.

Let us consider the following example:
\setlength{\arraycolsep}{1pt}
\[
\begin{array}{ccccc}
\To_{r_1} &    p   & \quad &\To_{r_2}&  p\\
\vee      &        & \quad &  \vee   &\\
\To_{r_3} & \neg p & \quad &\To_{r_4}& \neg p
\end{array}
\]

If the chain for $\neg p$ with rule $r_4$ is chosen as the winning chain, the
first approach would give $\{r_1 > r_3, r_4 > r_1, r_4 > r_2\}$ as an output,
erasing one priority and introducing two, while the second approach would
generate the following priority set: $\{r_3 > r_1, r_4 > r_2\}$, erasing
two priorities, and introducing two. If a new rule $r_i$ defeats $r_4$, it
is easy to see that in the first case we have to entirely revise the theory
(for example, let $r_3$ win against $r_1$ and $r_2$), while in the second case
we have only to introduce $r_3 > r_2$.

\medskip

\noindent \textbf{Instance $-\omega\non p \wedge +\sigma\non p$:} there exists
an undefeated chain supporting $\non p$. To revise the
theory, we have to choose one of them and, starting from $\non p$,
go back in the chain to the ambiguity point (where $P(i) = +\partial p_i
\wedge P(i+1)=-\partial p_{i+1}$ holds), and strengthen the chain by adding a
priority where a rule leading to an antecedent in the chain for $\non p$
and a rule for the opposite have the same strength.

\medskip

\noindent \textbf{Instance $-\omega\non p \wedge -\sigma\non p$:} every chain
supporting $\non p$ is defeated at least one time. A plausible solution could be to go back in the chain searching for the point where
$P(i) = +\sigma p_i$ and $P(i+1)=-\sigma p_{i+1}$. But this is not
enough to guarantee the chain to win. Let us consider the following example.

\begin{example}
Let $D$ be a theory having the following rules:
\setlength{\arraycolsep}{1pt}
\[
\begin{array}{cccccccc}
\stackrel{+\partial/-\partial}{\phantom{.}} &  &  &  & \stackrel{+\sigma/-\sigma}{\phantom{.}} &  &  & \\ 
\To_{r_1} &    a   & \To_{r_2} &  b & \To_{r_3} & c & \To_{r_4} & p\\
 & & & & \wedge & & & \\
\To_{r_5} &   \neg a   &  &   & \To_{r_6} & \neg c &  &
\end{array}
\]
It is easy to see that the sole condition of $r_3$ winning over $r_6$ is not sufficient: we
have also to introduce a priority between $r_1$ and $r_5$. Thus, we have
to act exactly as in the aforementioned case, with the solely difference that every
time a rule in the chain supporting $\non p$ is defeated, the priority rule
has to be inverted.
\end{example}

\subsection{Third canonical case: from $-\partial p$ to $+\partial p$}



For a proper analysis of this case, condition $-\partial\non p$
must hold since otherwise the case is analogous of the previous revision
from $+\partial q$ to $+\partial\non q$. Also, we do not take into
consideration the case where $-\Sigma p$ holds: if there are no chains
leading to $p$, then no revision to obtain $+\partial p$ is possible. The cases studied in this subsection
are reported in Figure \ref{figRevTreeToObtain_+Partial_p}.

\begin{figure}[htp]
\centering
\begin{tikzpicture}[scale=.7, font=\small]
  \node[shape=rectangle, text width=3cm, text centered] (root) at (0,0) {From $-\partial p \text{ to } +\partial p$
  ($-\partial\non p \wedge +\Sigma p$)};
  \node (+omega) at (-3,-1.5) {$+\omega p \wedge +\sigma p$};
  \node[shape=rectangle, text width=1.2cm, text centered] (-omega) at
  (3,-1.8) {$-\omega p$ ($-\omega\non p$)};
  \node (+s-o) at (1.5,-3.5) {$+\sigma p$};
  \node (-s-o) at (4.5,-3.5) {$-\sigma p$};
 \begin{scope}[>=stealth,thick,shorten >=2pt,shorten <=2pt]
  \draw (root) -- (+omega);
  \draw (root) -- (-omega);
  \draw (-omega) -- (+s-o);
  \draw (-omega) -- (-s-o);
 \end{scope}
\end{tikzpicture}
\caption{From $-\partial p$ to $+\partial p$: revision
cases.} \label{figRevTreeToObtain_+Partial_p}
\end{figure}
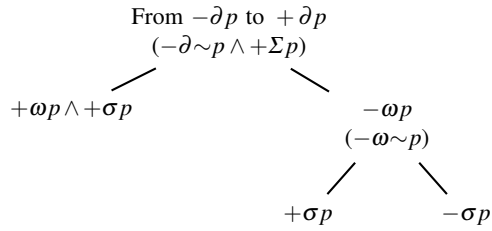
Notice that $+\omega p$ and $-\sigma p$ cannot hold at the same time: if all
the premises for $p$ are proved, then the chain fails on the last step, i.e.,
it has to be defeated by a firing rule for $\non p$. This would defeasibly
prove $\non p$, but this cannot happen since we have stated that
$-\partial\non p$ holds. Furthermore, $-\omega p$ implies that $-\omega\non p$
holds as well, since if it is not the case, we would have either $+\omega p$,
or $+\partial\non p$, which are both against the hypothesis.

\medskip

\noindent \textbf{Instance $+\omega p \wedge +\sigma p$:} we choose one of the
chains where condition $+\omega p \wedge +\sigma p$ holds, and introduce
as many priorities as the number of chains where $+\omega\non p$ holds.

\medskip

\noindent \textbf{Instance $-\omega p \wedge +\sigma p$:} this case is
analogous to instance $-\omega\non p \wedge +\sigma\non p$ of canonical case
from $+\partial p$ to $+\partial\non p$.

\medskip

\noindent \textbf{Instance $-\omega p \wedge -\sigma p$:} this case is
analogous to instance $-\omega\non p \wedge -\sigma\non p$ of canonical case
from $+\partial p$ to $+\partial\non p$.\\

\noindent We remark that conditions $\pm\sigma\non p$ are not useful 
for the revision process, since they do not give information whether the
changes will affect chains for $\non p$, or not. Example~\ref{ex:3can1} shows
that, even if $+\sigma\neg p$ holds, two distinct revisions exist: the first
involves the chain for $\neg p$ (introducing $r_1>r_3$),
the second does not (introducing $r_5>r_6$).

\begin{example}\label{ex:3can1}
Let $D$ be a theory having the following rules:
\[
\begin{array}{cccc}
\To_{r_1} &  a  & \To_{r_2} &  p\\
\To_{r_3} & \neg a  & \To_{r_4} & \neg p\\
 &  &  & \\
\To_{r_5} & b  & \To_{r_6} & p\\
\To_{r_6} & \neg b & &
\end{array}
\]
\end{example}

An analogous situation is proposed for $-\sigma\non p$ in Example~\ref{ex:3can1.1}.

\begin{example}\label{ex:3can1.1}
Let $D$ be a theory having the following rules:
\[
\begin{array}{cccccc}
\To_{r_1} &  a  & \To_{r_2} &  p\\
\To_{r_3} & \neg a  & \To_{r_4} & b & \To_{r_5} & \neg p\\
 &  & \wedge & & & \\
 &  & \To_{r_6} & \neg b & & \\
\To_{r_7} & c & \To_{r_8} & p & & \\
\To_{r_9} & \neg c &  & & &
\end{array}
\]
In here, two revisions exist: one introducing $r_1>r_3$, and the other one which introduces $r_7>r_9$.
\end{example}

\medskip

\noindent Notice that in all the canonical cases, the revision mechanism
guarantees that no cycle can be introduced. We can formulate this result,
which is a straightforward consequence of the case analysis presented here.

\begin{thm}\label{closure}
Revising a superiority relation generates a superiority relation.
\end{thm}

%% file: AGMComparisonPRDL.tex
\section{AGM postulates analysis}\label{sec:AGMComparisonPRDL}

The aim of this section is to study the canonical cases described in
Section~\ref{sec:PrefDefRevisionPRDL} from the point of view of the AGM
approach. In particular, we try to relate the canonical cases with the types
of theory change analysed by Alchourr{\'o}n, G{\"a}rdenfors and Makinson in
their seminal work \cite{AGM} as far as possible. Afterwards, we focus on
understanding the meaning of the various AGM postulates in terms of the
changes we proposed. This allows us then to identify which of the AGM
postulates are satisfied by our canonical cases.

This research issue is motivated, as introduced in
Section~\ref{sec:Introduction}, by the fact that the AGM postulates analysis
in non-monotonic formalisms is still controversial, and thus open to
discussion.

We recall that Delgrande proposed an approach to belief revision of logic
programs under answer set semantics that is fully compliant with the base AGM
postulates for revision \cite{DelgrandeSTW08}. He also claims in a later work
\cite[p. 568]{Delgrande10} that the third and fourth postulates for
belief revision are not appropriate for belief revision of non-monotonic
theories, and thus are ignored in his work. However, we are going to argue
that these two postulates can be adopted in our approach, which suggests that 
the question whether the AGM postulates are suitable for non-monotonic
reasoning is still open.



In the remainder we assume that the reader is
familiar with the terminology used in the AGM framework, in particular with
the notions of \emph{belief}, \emph{belief set}, and
\emph{theory}\footnote{Notice that in our framework the hypothesis of
\emph{completeness} of a theory does not hold in general, as it could be the
case that in a defeasible theory neither $+\partial p$, nor $+\partial \non p$
is derivable.}. 

To adjust the AGM framework in the perspective of preference revision, we
first rephrase the concept of extension into that of \emph{belief set}
corresponding to a defeasible theory.

\begin{defn}\label{def:beliefset}
Let $D = (F, R, >)$ be a defeasible theory. Then
\begin{displaymath}
BS(D) = BS^{+\partial}(D) \cup BS^{-\partial}(D)
\end{displaymath}
is the \emph{belief set} of $D$, where
\begin{eqnarray}
BS^{+\partial}(D) & = & \{p \, | \, p\textrm{ is a literal appearing in }D\textrm{ and }D\vdash +\partial p\}, \nonumber\\
BS^{-\partial}(D) & = & \{p \, | \, p\textrm{ is a literal appearing in }D\textrm{ and }D\vdash -\partial p\}. \nonumber
\end{eqnarray}
\end{defn}
We also state that when a literal $p$ is believed, $p \in K$ in AGM
notation, then $p \in BS^{+\partial}(D)$. Conversely, if a literal is not
believed, i.e., $p \notin K$, then $p \in BS^{-\partial}(D)$\footnote{Notice
that it is possible that a literal $p$ and its complement do not belong to
$BS(D)$. For example, consider the theory consisting only of $p \To p$ and
$\neg p \To \neg p$. In this theory none of $\pm\partial p$ and $\pm\partial
\neg p$ is provable.}. Intuitively, the idea is that if we prove $+\partial
p$ then we believe in $p$, and if we prove $-\partial p$ then we do not
believe in $p$.

\vspace{5mm}

\noindent In the remainder of the section, we relate the AGM operators of
\emph{contraction}, \emph{expansion}, and \emph{revision}, and then
reframe the corresponding postulates of AGM in the terminology of defeasible
theories (in Subsections \ref{subsec:AGMContr}--\ref{subsec:AGMRev}).

\vspace{2mm}

\emph{Belief contraction} is the process of rationally removing from a belief set $K$
a certain belief $\psi$ previously in the set. From the point of view of
Defeasible Logic, by Definition \ref{def:beliefset}, a defeasible theory $D = (F, R, >)$
where $D\vdash +\partial p$ (i.e., $p \in BS^{+\partial}(D)$) is modified such
that $-\partial p$ holds in the contracted theory (denoted by $D^-_p$) after
the process (i.e., $p \in BS^{-\partial}(D^-_p)$). For the above reasoning, it
seems reasonable to argue that the process of belief contraction as formalised
in AGM approach corresponds to our first canonical case, i.e., from $+\partial
p$ to $-\partial p$. If we consider a set of literals $C = \{p_1, \ldots,
p_n\}$, we define the contracted theory $D^-_C$ as the theory where for each
$p_i \in C$, $p_i \in BS^{-\partial}(D^-_C)$.


\emph{Belief revision} is the process of rationally deleting a certain belief
$\psi$ from a belief set $K$ and adding its opposite. From the point of
view of Defeasible Logic, by Definition \ref{def:beliefset}, a defeasible theory $D=(F, R,
>)$ where $D \vdash +\partial \non p$ (i.e., $\non p \in BS^{+\partial}(D)$
and $p \in BS^{-\partial}(D)$) is modified such that $+\partial p$ holds in
the revised theory (denoted by $D^*_p$) after the process (i.e., $p \in
BS^{+\partial}(D^*_p)$). Remember that in Defeasible Logic $\non p$ now belongs to
$BS^{-\partial}(D^*_p)$. For the above reasoning, it seems reasonable to argue
that the process of belief revision as formalised in AGM approach corresponds
to our second canonical case, i.e., from $+\partial \non p$ to $+\partial p$.


\emph{Belief expansion} is the process of adding a certain belief $\psi$ to a
belief set $K$. It is possible to consider two interpretations of the
expansion process: the first where we simply force the belief in, the second
where a belief is added if the opposite is not believed. Our third canonical
case, i.e., from $-\partial p$ to $+\partial p$, follows the second strategy.
Therefore, from the point of view of Defeasible Logic, by Definition \ref{def:beliefset}
this process describes the case of an initial defeasible theory $D = (F, R,
>)$ where $D \vdash -\partial p$ and $D \vdash -\partial \non p$ (i.e., $p,
\non p \in BS^{-\partial}(D)$) hold being modified such that $+\partial p$
holds in the expanded theory (denoted by $D^+_p$) after the process (i.e., $p
\in BS^{+\partial}(D^+_p)$). Remember that in Defeasible Logic $\non p$ still belongs to
$BS^{-\partial}(D^+_p)$. If we consider a set of literals $C = \{p_1, \ldots,
p_n\}$, we define the expanded theory $D^+_C$ as the theory where for each
$p_i \in C$, $p_i \in BS^{+\partial}(D^+_C)$.

\subsection{Preference contraction}\label{subsec:AGMContr}

Throughout this subsection, we assume that $D \vdash +\partial p$ for a
literal $p$ in $D$.

\vspace{5mm}

\noindent The first postulate in AGM belief contraction states that when a
belief set is contracted by a sentence $p$, the outcome should be logically
closed. In Defeasible Logic, we distinguish between a theory (i.e., a set of rules), and its
extension (i.e., its set of conclusions). In general, given an extension in
Defeasible Logic, there are multiple (possibly not equivalent) theories that generate the
extension. This means that in AGM there is no difference to contract a theory
or its base, while it is not the case in Defeasible Logic.
\begin{flalign}
	\tag{$K \dot{-} 1$} D^-_p \text{ is a theory.} &&
\end{flalign}
As preference contraction acts only on the superiority relation, to ensure
that a contraction operation satisfies the postulate, we only have to check if
the operation itself does not create a cycle in the superiority relation. This is guaranteed by the following proposition.

\begin{prop}\label{prop:contrNOCylcles}
	Given a defeasible theory $D=(F,R,>)$, if $D'=(F,R,>')$ is obtained from $D$ by erasing preference tuples from $>$, then $>'$ is acyclic.
\end{prop}
\begin{proof}
	By contradiction, let us suppose that there is a cycle in $>'$. Since, by hypothesis, $>'$ is obtained from $>$ by simply removing preference tuples, then each element of $>'$ is an element of $>$ and the cycle in $>'$ is also in $>$, against the hypothesis. 
\end{proof}

\noindent The idea of the second AGM postulate for belief contraction is that
a contraction removes beliefs, thus a contraction operation always produces a
belief set smaller than the original. AGM focuses only on ``positive''
beliefs. However, in our framework we have two possible types of defeasible
conclusions (as it turns out also by Definition \ref{def:beliefset}), thus we
have to check the relationships between the elements of the defeasible belief
sets before and after the operation. In particular, since we remove a belief,
then the set of formulae believed should be smaller after the contraction;
this also means that the set of formulae we do not believe is increased by the
formula we contract. As a consequence, the second postulate must be rewritten
taking into account the two parts.
\begin{flalign}
	\tag{$K \dot{-} 2$} BS^{+\partial}(D^-_p) \subseteq
	BS^{+\partial}(D) \text{ and } BS^{-\partial}(D^-_p)\supseteq BS^{-\partial}(D). &&
\end{flalign}
This postulate cannot be adopted in our framework because it contradicts the
sceptical non-monotonic nature of Defeasible Logic. To see this, suppose that we know $a$
and we have the rules $\To p$ and $a \To \neg p$. Then $a$ is
sceptically
provable, and $p$ is not. But if we decide to contract $a$, then $p$ becomes
defeasibly provable, thus we have $p\in BS^{-\partial}(D)$ but $p\in
BS^{+\partial}(D^-_p)$.\footnote{In general, making a literal $p$ no longer
defeasibly provable does not imply that $-\partial p$ holds after the revision
process. For example, consider the theory ${} \To_{r} p$ and ${} \neg p\To_{s}
\neg p$. The only way to prevent $+\partial p$ is to impose $s>r$, but in the
resulting theory none of $+\partial p$ and $-\partial p$ holds (same for $\neg
p$). Notice that in this case the conditions for our canonical cases to
succeed do not hold.}
Notice that this behaviour is not confined
to the specifics of Defeasible Logic, but holds in any sceptical non-monotonic formalism.

\vspace{5mm}

\noindent The third postulate of AGM considers the case when a belief
$\psi$ is not in the initial belief set. As we have already discussed, AGM
focuses on a classical notion of consequence relation, thus if $\psi$ is
not a consequence of the theory, then there is no reason to change anything at
all. In Defeasible Logic, this corresponds to not being able to prove $p$.
Accordingly, we can state that $p \in BS^{-\partial}(D)$.
\begin{flalign}
	\tag{$K \dot{-} 3$} \text{If } p \in BS^{-\partial}(D)
	\text{ then } BS(D^-_p) = BS(D). &&
\end{flalign}
Since we want to obtain a theory where $-\partial p$ holds and by hypothesis
$p \in BS^{-\partial}(D)$, then the postulate trivially holds.

\vspace{5mm}

\noindent The fourth AGM postulate states that the only literals that are
immutable in the contraction process are tautologies. Defeasible Logic does not have logical
connectives, thus it is not possible to have tautologies in the classical
sense. Nevertheless, the concept of tautology is that of a statement that
cannot ever be refuted, i.e., it is true in every interpretation. In classical
logic, an interpretation is an assignment of truth values to the propositional
atoms, while in Defeasible Logic corresponds to consider a particular set of propositional
atoms as factual knowledge. In the context of this paper, where we assume that
the set of facts cannot be changed, the closest thing to an interpretation is
an assignment of the superiority relation. We give the formulation of the
success postulate for contraction using the contrapositive.
\begin{flalign}
	\tag{$K \dot{-} 4$} \text{If } p \in BS(D^-_p) \text{ then } D \vdash +\Delta p. &&
\end{flalign}

\noindent The concept of strict derivation embodied by $+\Delta$ cannot fully
capture the notion of tautology as a non-refutable statement, since the proof
tag $+\varphi$ indeed denotes the presence of a supporting chain made of
elements for which there are no rules for the opposite, and so \emph{de facto}
a non-refutable argument obtained from defeasible rules.

Thus, it seems reasonable to reformulate the success postulate for
contraction as follows.
\begin{flalign}
	\tag{$K \dot{-} 4'$} \text{If } p \in BS(D^-_p) \text{ then } D \vdash +\varphi p. &&
\end{flalign}

\noindent Even this version of the postulate does not hold in Defeasible Logic. Indeed,
there exist situations where there is a proof for $p$ and it is not possible
to change the theory in order to make $p$ no longer provable, even if there
are opposite literals of some elements for every chain supporting $p$. A
simple situation is to take a tautologous 3-SAT formula and to generate its
$\Gamma$-transformation (see Definition~\ref{def:tranformation}). There are literals in the theory obtained that cannot
be contracted. However, there are more cases.

\begin{example}\label{ex:Taut}
Let $D$ be a defeasible theory with the following set of rules:
\setlength{\arraycolsep}{1pt}
\[
\begin{array}{cccccc}
\To_{r_1}  &       l  &  \To_{r_2}  &  \neg a  &             &    \\
           &          &  \To_{r_3}  &       a  &  \To_{r_4}  &  p \\
           &          &  \To_{r_5}  &       b  &  \To_{r_6}  &  p \\
\To_{r_7}  &  \neg l  &  \To_{r_8}  &  \neg b.  &             &    \\
\end{array}
\]
To contract $p$, we must block both the chains proving $p$. But, in order to
do so, we should have that $D \vdash +\partial l$ as well as $D \vdash
+\partial\neg l$. This is not possible since $D$ is consistent.
\end{example}

Unfortunately, the rule pattern shown in Example~\ref{ex:Taut} is not a
sufficient condition to reframe the postulate $(K \dot{-} 4')$. Indeed, as
Example~\ref{ex:ContrTaut} shows, it is possible to find counter-examples
where $p$ can be contracted, as well as counter-examples to counter-examples
(we refer to Example~\ref{ex:ContrContrTaut}) where, by extending the theory
of Example~\ref{ex:ContrTaut} with rules $\set{r_{19},\dots,r_{25}}$, the
contraction of $p$ becomes, again, not possible.

\begin{example}\label{ex:ContrTaut}
Let $D$ be a defeasible theory with the following set of rules:
\setlength{\arraycolsep}{1pt}
\[
\begin{array}{cccccc}
           &          &  \To_{r_1}     &       a  &  \To_{r_2}   &  p \\
           &          &  \To_{r_3}     &       b  &  \To_{r_4}   &  p \\
           &          &  \To_{r_{5}}   &       c  &  \To_{r_{6}} &  p \\

\To_{r_7}  &       l  &  \To_{r_8}     &  \neg a  &   &  \\
\To_{r_9}  &  \neg l  &  \To_{r_{10}}  &  \neg b  &   &  \\

\To_{r_{11}}  &       m  &  \To_{r_{12}}  &  \neg b  &  &  \\
\To_{r_{13}}  &  \neg m  &  \To_{r_{14}}  &  \neg c  &  &  \\

\To_{r_{15}}  &       n  &  \To_{r_{16}}  &  \neg c  &  &  \\
\To_{r_{17}}  &  \neg n  &  \To_{r_{18}}  &  \neg a.  &  &  \\
\end{array}
\]
To contract $p$, we must block derivations of $+\partial a$, $+\partial b$ and
$+\partial c$. This can be obtained by adding the following tuples to the
superiority relation: $(r_{7},r_{9})$, $(r_{11},r_{13})$ and
$(r_{15},r_{17})$.
\end{example}

\begin{example}\label{ex:ContrContrTaut}
\setlength{\arraycolsep}{1pt}
\[
\begin{array}{cccccc}
           &          &  \To_{r_{19}}     &       e  &  \To_{r_{20}}   &  p \\
           &          &  \To_{r_{21}}     &       f  &  \To_{r_{22}}   &  p \\

           &       n  &  \To_{r_{23}}  &  \neg e  &  &  \\
           &  \neg n  &  \To_{r_{24}}  &  \neg f &  &  \\
           &  \neg m  &  \To_{r_{25}}  &  \neg f.  &  &  \\
\end{array}
\]
To contract $p$, we must now block derivations also of $+\partial e$, and
$+\partial f$. Derivation of $e$ can be blocked only if we prove the
antecedent of $r_{23}$, that is $n$ (the derivation of $c$ is blocked as
well). This implies that the derivation of $f$ is blocked only if
$+\partial\neg m$ holds (the only antecedent of rule $r_{25}$). We can now
operate only on the provability of either $l$, or $\neg l$. In both cases, one
between $a$ or $b$ cannot be refuted.
\end{example}

\noindent In Subsection~\ref{subsec:np_completeness} we have shown that, in
general, revising a defeasible theory using only the superiority relation is
an NP-complete problem. This suggests that there might not be a simple
condition, based on proof tags, that can be computed in polynomial time and
also guarantees a successful contraction.

\vspace{5mm}

\noindent The fifth AGM postulate states that contracting, and then expanding
by the same belief $\psi$ will give back at least the initial theory.
\begin{flalign}
	\tag{$K \dot{-} 5$} \text{If } p \in BS^{+\partial}(D) \text{ then } BS(D) \subseteq BS((D^-_p)^+_p). &&
\end{flalign}
This postulate cannot be adopted since, once the contracted theory has been
obtained, the backward step does not uniquely correspond to expanding the
obtained theory by the same literal, as the following example shows.

\begin{example}\label{ex:Contr5}
Let $D$ be a defeasible theory having the following rules:
\setlength{\arraycolsep}{1pt}
\[
\begin{array}{cccc}
\To_{r_1}     &      a        &  \To_{r_2}   &  p \\
 \vee           &                &                   &     \\
\To_{r_3}     & \neg a    &                    &     \\
                  &                &                    &     \\ 
\To_{r_4}     &      b        &  \To_{r_5}   &  p \\
\To_{r_6}     & \neg b    &                    &     \\
\end{array}
\]
If we contract $D$ by $p$, in the contracted theory the preference $r_1 > r_3$
is no longer present. If we now expand $D^-_p$, one solution is the initial
theory, but also the theory containing the preference $r_4 > r_6$ is another
valid solution.
\end{example}
Nevertheless, if all operations in the contraction process can be traced, then
we can easily backtrack and obtain the initial theory, satisfying the
postulate.

\vspace{5mm}

\noindent The sixth AGM postulate, also known as the postulate \emph{of the
irrelevancy of syntax}, states that if two beliefs $\psi$ and $\chi$ are
logically equivalent, then contracting by $\psi$ and contracting by $\chi$
produce the same result.
\begin{flalign}
	\tag{$K \dot{-} 6$} \text{If } \vdash p \equiv q \text{ then }
	BS(D^-_p) = BS(D^-_q). &&
\end{flalign}
In the framework of Defeasible Logic, the language is restricted to literals, thus two
elements $p$ and $q$ are equivalent only if they represent the same literal.
For this reason, the sixth postulate straightforwardly follows.

\vspace{5mm}

\noindent The seventh and the eighth postulate are best understood if seen in
combination. They essentially relate two individual contractions with respect
to a pair of sentences $\psi$ and $\chi$, with the contraction of their
conjunction $\psi \wedge \chi$. As already stated, in Defeasible Logic there are no
logical connectives, and a conjunction of literals is equivalent to the set of
the same literals; the same reasoning used to introduce postulate $(K \dot{-}
2)$ applies here. Thus, the two postulates can be rewritten as
\begin{flalign}
	\tag{$K \dot{-} 7$}
	\begin{array}[t]{l}
		BS^{+\partial}(D^-_p) \cap BS^{+\partial}(D^-_q) \subseteq BS^{+\partial}(D^-_{p,q}) \text{ and}\\
		BS^{-\partial}(D^-_p) \cap BS^{-\partial}(D^-_q) \supseteq BS^{-\partial}(D^-_{p,q}).
	\end{array} &&
\end{flalign}
\begin{flalign}
	\tag{$K \dot{-} 8$}
	\begin{array}[t]{l}
		\text{If } p \in BS^{-\partial}(D^-_{p,q}) \text{ then } BS^{+\partial}(D^-_{p,q}) \subseteq BS^{+\partial}(D^-_p) \text{ and}\\
		BS^{-\partial}(D^-_p) \subseteq BS^{-\partial}(D^-_{p,q}).
	\end{array} &&
\end{flalign}
Postulates $(K \dot{-} 7)$ and $(K \dot{-} 8)$ do not hold for the same reason
formulated for postulate $(K\dot{-}2)$. The following example shows the truth
of the statement for both of them.

\begin{example}\label{ex:Contr78}
Let $D$ be a defeasible theory having the following rules:
\setlength{\arraycolsep}{1pt}
\[
\begin{array}{cccccccc}
\To_{r_3}    &   \neg a    &                    &             &   &                          &                  &      \\
 \wedge      &                &                    &             &    &                         &                  &      \\
\To_{r_1}     &      a        &  \To_{r_2}   &  c         &    &                         &                  &      \\
                   &                &      \vee       &             &    &                         &              &   \\  
                   &                &  \To_{r_3}   & \neg  c  &\qquad &\neg c, \neg d &    \To_{r_4}   &   p \\ 
                   &                &  \To_{r_5}   & \neg  d  &   &                          &    \vee      & \\ 
                   &                &     \wedge   &              &   &                          &  \To_{r_6}   &  \neg p  \\ 
  \To_{r_7}  &      b        &  \To_{r_8}   &     d        &   &                          &                  &  \\ 
 \vee           &                &                    &             &    &                         &                  &      \\                                                       
\To_{r_9}    &   \neg b    &                    &             &   &                          &                  &      \\
\end{array}
\]
In this theory, we have $BS^{+\partial}(D) = \{a, b, c, d, \neg p\}$, and
$BS^{-\partial}(D) = \{\neg a, \neg b, \neg c, \neg d, p\}$. Let us contract
$D$ by literal $a$ and by literal $b$ (where the contractions are minimal with
respect to the changes in the superiority relation) obtaining:
\begin{eqnarray}
BS^{+\partial}(D^-_a) &=& \{b, \neg c, d, \neg p\}\nonumber\\
BS^{+\partial}(D^-_b) &=& \{a, c, \neg d, \neg p\}\nonumber\\
BS^{-\partial}(D^-_a) &=& \{a, \neg a, \neg b,  c, \neg d, p\}\nonumber\\
BS^{-\partial}(D^-_b) &=& \{\neg a, b, \neg b,  \neg c, d, p\}.\nonumber
\end{eqnarray}
The respectively intersections are: 
\begin{eqnarray}
BS^{+\partial}(D^-_a) \cap BS^{+\partial}(D^-_b) &=& \{\neg p\}\nonumber\\
BS^{-\partial}(D^-_a) \cap BS^{-\partial}(D^-_b) &=& \{\neg a, \neg b, p\}.\nonumber
\end{eqnarray}
We can now contract $a$ and $b$ simultaneously, and obtain 
\begin{eqnarray}
BS^{+\partial}(D^-_{a, b}) &=& \{\neg c, \neg d, p\}\nonumber\\
BS^{-\partial}(D^-_{a, b}) &=& \{a, \neg a, b, \neg b,  c, d, \neg p\}\nonumber
\end{eqnarray}
proving our claim.
\end{example}

Throughout postulates $(K \dot{-} 1)$ to $(K \dot{-} 8)$ we took care of the
transition effects of the contraction process, due to the specific nature of
positive and negative beliefs in Defeasible Logic. However, for each postulate this
specificity has no effect. In fact, what can be claimed for contractions in
$BS^{+\partial}$ extends to $BS^{-\partial}$, and vice versa.

For the sake of completeness, we apply the same care to expansion and revision
cases further on. As it will be clear at the end of each analysis, analogous
conclusions about the redundancy are derived.

\subsection{Preference revision}\label{subsec:AGMRev}

Throughout this subsection, we assume that $D \vdash +\partial \non p$ and $D
\vdash +\Sigma p$ for a literal $p$ in $D$.

\vspace{5mm}

\noindent The first AGM postulate for revision states that the revision
process has to preserve the logical closure of the initial theory.
\begin{flalign}
	\tag{$K * 1$} D^*_p \text{ is a theory}. &&
\end{flalign}
The reasoning is the same made for the first postulate for contraction and
expansion, and it assures that also $(K*1)$ is satisfied in our framework.

\vspace{5mm}

\noindent The second AGM postulate for revision captures the most general
interpretation of theory change; the new information $\psi$ is always
included in the new belief set, even if $\psi$ is self-inconsistent, or
contradicts some belief of the initial theory. Henceforth, the complete
reliability of the new information is always assumed.
\begin{flalign}
	\tag{$K * 2$} p \in BS^{+\partial}(D^*_p). &&
\end{flalign}
As by definition of our second canonical case, literal $p$ is forced to be
defeasibly proved after the process, provided that preconditions $+\partial
\non p$ and $+\Sigma p$ hold, the postulate is clearly satisfied.

\vspace{5mm}

\noindent The third and the fourth postulates of AGM revision explain the
relationship between the revision and the expansion processes. The
quintessential meaning is that they are independent by operators implementing
them.
\begin{flalign}
	\tag{$K * 3$} BS^{+\partial}(D^*_p) \subseteq BS^{+\partial}(D^+_p). &&
\end{flalign}
\begin{flalign}
	\tag{$K * 4$} \text{If } \non p \in BS^{-\partial}(D)
	\text{ then } BS^{+\partial}(D^+_p) \subseteq BS^{+\partial}(D^*_p). &&
\end{flalign}
Both the first two canonical cases, starting from an initial theory and
considering a literal $p$, operate to obtain a final theory where $+\partial
p$ holds. What we have to care about, however, are the preconditions for which
these two canonical cases apply. The third postulate essentially states that
every belief which can be derived revising a theory by a belief $\psi$ can be
also obtained expanding the same initial theory with respect to the same
belief. This statement is perfectly allowed in our framework; the case where
both revision and expansion can apply is when $+\partial \non p$ (and hence
$-\partial p$) holds in the initial theory; in this situation, the two
processes behave in the same manner, i.e., they calculates the same
extensions. However, if we regard at proper expansion, i.e., when condition
$-\partial \non p$ holds, then it is easy to see that the preconditions for
expansion and revision are mutually exclusive and can not be applied at the
same time.



\vspace{5mm}

\noindent The fifth AGM postulate states that the result of a revision by a
belief $\psi$ is the absurd belief set iff the new information
is in itself inconsistent.
\begin{flalign}
	\tag{$K * 5$} \text{If } p \text{ is consistent then } BS^{+\partial}(D^*_p) \text{ is also consistent}. &&
\end{flalign}
Since by definition a literal $p$ is always consistent, and the extension of a
consistent theory is also consistent, the postulate is trivially satisfied.

\vspace{5mm}

\noindent The sixth AGM postulate for revision follows the same idea of ($K
\dot{-} 6$): the syntax of the new information has no effect on the revision
process, all that matters is its content. Again, it has a natural counterpart
in Defeasible Logic.
\begin{flalign}
	\tag{$K * 6$} \text{If } \vdash p \equiv q \text{ then }
	BS^{+\partial}(D^*_p) = BS^{+\partial}(D^*_q). &&
\end{flalign}
The reasoning is the same exploited in the counterpart postulate for
contraction, and the postulate is straightforward.

\vspace{5mm}

\noindent The seventh and the eight postulate of AGM revision cope with the
revision process with respect to conjunctions of literals. In the classical
AGM framework, the principle of minimal change takes an important role in the
formulation of these postulates. The revision with both $\psi$ and $\chi$
should correspond to a revision of the theory with $\psi$ followed by an
expansion by $\chi$, provided that $\chi$ does not contradict the beliefs in
the theory revised by $\psi$.
\begin{flalign}
	\tag{$K * 7$} BS^{+\partial}(D^*_{p, q}) \subseteq
	BS^{+\partial}((D^*_p)^+_q) \text{ and } BS^{-\partial}((D^*_p)^+_q) \subseteq BS^{-\partial}(D^*_{p, q}). &&
\end{flalign}
\begin{flalign}
	\tag{$K * 8$}
	\begin{array}[t]{l}
		\text{If } \neg q \in BS^{-\partial}(D^*_p) \text{ then } BS^{+\partial}((D^*_p)^+_q) \subseteq BS^{+\partial}(D^*_{p,q}) \text{ and}\\
		BS^{-\partial}(D^*_{p, q}) \subseteq BS^{-\partial}((D^*_p)^+_q).
	\end{array} &&
\end{flalign}
Again, since the sceptical nature of Defeasible Logic, these postulates cannot be satisfied.
The following example gives a specific case that falsifies them.

\begin{example}\label{ex:Rev78}
Let $D$ be a defeasible theory having the following rules:
\setlength{\arraycolsep}{1pt}
\[
\begin{array}{cccc}
\To_{r_1}     &  \neg a    &				       &			   \\
\To_{r_2}     &          a    & \To_{r_3}      &  		   p  \\
                   &                & \To_{r_4}      &  \neg p  \\
\To_{r_5}     &          b    &                     &  		       \\
\To_{r_6}     &  \neg b    &				       &			   \\
                   &          b    & \To_{r_7}      &  		   p  \\
                   &          b    & \To_{r_8}      &  		   q   \\
                   &                & \To_{r_9}      &  \neg q   \\
\To_{r_{10}} &          c    & \To_{r_{11}}  &   q          \\
\end{array}
\]
Given theory $D$, we have $BS^{+\partial}(D) = \{\neg p\}$, while all other
literals are in $BS^{-\partial}(D)$. Revising $D$ for $p$ and $q$, one
possible theory is $D^*_{p, q}$, obtained operating through the provability of
literal $b$ and adding the following superiorities: $r_5 > r_6$, $r_7 > r_4$,
and $r_8 > r_9$. The resulting $BS^{+\partial}(D^*_{p, q})$,
$BS^{-\partial}(D^*_{p, q})$ are respectively $\{b, c, p, q\}$, and $\{a, \neg
a, \neg b, \neg p, \neg q\}$.

Now, let us consider the revision just by $p$. A possible solution is $D^*_p$
such that $BS^{+\partial}(D^*_{p}) = \{a, c, p\}$, and $BS^{-\partial}(D^*_p)
= \{\neg a, b, \neg b, \neg p, q, \neg q\}$. In this case the revision process
acts on the provability of literal $a$, adding $r_2 > r_1$, and $r_3 > r_4$.

If we expand $D^*_p$ by $q$, one possible resulting theory is $(D^*_p)^+_q$
(the expansion process now operates on the provability of literal $c$ adding
$r_{11} > r_9$) where $BS^{+\partial}((D^*_p)^+_q) = \{a, c, p, q\}$, and
$BS^{-\partial}((D^*_p)^+_q) = \{\neg a, b, \neg b, \neg p, \neg q\}$. The
intersection of $D^*_{p, q}$ and $(D^*_p)^+_q$ is not empty, but neither
theory is contained in the other.
\end{example}

\subsection{Preference expansion}\label{subsec:AGMExp}

Throughout this subsection, we assume that for a literal $p$ in $D$ both $D
\vdash -\partial p$ and $D \vdash -\partial \non p$. Moreover, $+\Sigma p$
holds.

\vspace{5mm}

\noindent The first AGM postulate for expansion states that if a theory is
expanded with a belief $\psi$, then the resulting theory is the logical
closure of the initial theory.
\begin{flalign}
	\tag{$K + 1$} D^+_p \text{ is a theory}. &&
\end{flalign}
The same idea for postulate ($K \dot{-} 1$) can be exploited, thus the
postulate is clearly satisfied.

\vspace{5mm}

\noindent The second AGM postulate for expansion assures that a belief
$\psi$ for which the expansion is performed always belongs to the belief
set of the resulting theory.
\begin{flalign}
	\tag{$K + 2$} p \in BS^{+\partial}(D^+_p). &&
\end{flalign}
By the hypotheses given at the beginning of this subsection, the postulate
trivially holds since the expansion process forces literal $p$ to be
defeasibly proved.

\vspace{5mm}

\noindent The joint formulation of the third and the fourth AGM postulates for
expansion states that if a belief is already present in the initial belief
set, then an expansion process lets the theory unchanged.
\begin{flalign}
	\tag{$K + 3$} BS^{+\partial}(T) \subseteq
	BS^{+\partial}(T^+_p) \text{ and } BS^{-\partial}(T^+_p) \subseteq
	BS^{-\partial}(T). &&
\end{flalign}
\begin{flalign}
	\tag{$K + 4$}  \text{If } p \in BS^{+\partial}(T)
	\text{ then } BS^{+\partial}(T^+_p) \subseteq BS^{+\partial}(T) \text{ and } BS^{-\partial}(T) \subseteq BS^{-\partial}(T^+_p). &&
\end{flalign}
Since we aim at obtaining a theory where $+\partial p$ holds, and by
hypothesis $p \in BS^{+\partial}(T)$, the postulates seen together trivially
hold given that, by definition, its premise does not hold.

\vspace{5mm}

\noindent The fifth AGM postulate states that if a belief set is contained in
another one, then the expansion of both sets wrt the same belief preserves the
relation of inclusion.
\begin{flalign}
	\tag{$K + 5$} \text{If } BS^{+\partial}(D) \subseteq
	BS^{+\partial}(D') \text{ then } BS^{+\partial}(D^+_p) \subseteq
	BS^{+\partial}(D'^+_p). &&
\end{flalign}
Also in this case, because of the sceptical non-monotonic nature of Defeasible Logic, this
postulate can not be satisfied, as already pointed out in the explanation of
Postulate $(K\dot{-}2)$.

\vspace{5mm}

\noindent Non-monotonic formalisms derive conclusions that are tagged. The
specific nature of this tagging is that it makes the notion of minimality for
a set of conclusions useless. We can consider minimality only for one given
tag, and not for all tags (this is particularly obvious for formalisms where
tags of formulae are interdependently defined). Thus, the idea of ``smallest
resulting set'' is meaningless in non-monotonic systems. The sixth AGM
postulate assures the minimality of the expanded belief set.
\begin{flalign}
	\tag{$K+6$}
	\begin{array}[t]{l}
		\text{Given a theory } D \text{ and a belief } p\text{, }\\
		BS(D^+_p) \text{ is the smallest belief set satisfying }
		(K+1)-(K+5).
	\end{array} &&
\end{flalign}
In the perspective of non-monotonic reasoning, the operation of expanding a
defeasible theory to prove a literal $p$ (only changing the preference
relation) necessarily falsifies some other literals, previously provable in
the initial theory.

\subsection{Preference identities}

In AGM framework a process that defines revision in terms of expansion is
available, suggested by Isaac Levi \cite{levi}. The idea is that to revise a
theory $D$ by a belief $\psi$ we may firstly contract $D$ by $\neg \psi$
in order to remove any information that may contradict $\psi$, and then
expand the resulting theory with $\psi$. This is know as the \emph{Levi
identity}, which can be rewritten using our terminology as:
\begin{flalign}
\tag{LI} BS(D^*_p) = BS((D^-_{\neg p})^+_p). &&
\end{flalign}
The following example shows that the Levi identity does not hold in our
framework.
\begin{example}\label{ex:IdLevi}
Let $D$ be a defeasible theory having the following rules:
\setlength{\arraycolsep}{1pt}
\[
\begin{array}{cccc}
\To_{r_1}     &      a        &  \To_{r_2}   &          p    \\
                   &                &                   &                \\
\To_{r_3}     & \neg a    &                    &                \\
                  &                &   \To_{r_4}   &  \neg p    \\ 
\To_{r_5}     &      b        &  \To_{r_6}   &          p   \\
\To_{r_7}     & \neg b    &                    &               \\
\end{array}
\]
If we revise $D$ by $p$, a possible solution is $D^*_p$ such that
$BS^{+\partial}(D^*_p) = \{a, p\}$, and $BS^{-\partial}(D^*_p) = \{\neg a, b,
\neg b, \neg p\}$. Now, contracting $D$ by $\neg p$ can lead to $D^-_{\neg p}$
with $BS^{+\partial}(D^-_{\neg p}) = \{b\}$, and $BS^{-\partial}(D^-_{\neg p})
= \{a, \neg a, \neg b, p, \neg p\}$. If we expand $D^-_{\neg p}$ by $p$, we
obtain $(D^-_{\neg p})^+_p$ with $BS^{+\partial}((D^-_{\neg p})^+_p) = \{b,
p\}$, and $BS^{-\partial}((D^-_{\neg p})^+_p) = \{a, \neg a, \neg b, \neg
p\}$.
\end{example}
The Levi identity does not hold as our revision procedure
concerns the reasons why one belief is obtained and not only whether we have
one belief. Thus when there are multiple reasons to justify one belief, it is
possible to contract the theory in multiple ways, and similarly to expand it
in multiple ways, and the changes for the contractions are not necessarily the
`opposite' of those for contraction.

As Levi Identity relates the revision process in terms of expansion, Harper
proposed a method to obtain the contraction using revision \cite{Gardenfors};
the underlying idea is that a theory $D$ contracted by a belief $\psi$ is
equivalent to the theory containing only the information that remain unchanged
during the process of revising $D$ by $\neg \psi$. In our terms, the
\emph{Harper Identity} can be rewritten as
\begin{flalign}
\tag{HI} BS(D^-_p) = BS(D^*_{\neg p}) \cap BS(D). &&
\end{flalign}
Harper Identity does not hold for the operations we defined in this paper.
Example \ref{ex:IdHarper} provides a counter-example to it.
\begin{example}\label{ex:IdHarper}
Let $D$ be a defeasible theory having the following rules:
\setlength{\arraycolsep}{1pt}
\[
\begin{array}{cccc}
\To_{r_1}  &      p   &  \To_{r_2}  &       q  \\
           &          &    \vee     &          \\
\vee       &          &   \To_{r_3} &  \neg q  \\
           &          &    \wedge   &          \\
\To_{r_4}  &  \neg p  &  \To_{r_5}  &      q        
\end{array}
\]
The initial belief set is $BS^{+\partial}(D) = \{p, q\}$ and
$BS^{-\partial}(D) = \{\neg p, \neg q\}$. If we contract $D$ by $p$, we obtain
a theory $D^-_p$ such that $BS^{+\partial}(D^-_p)$ is $\{\neg q\}$ and
$BS^{-\partial}(D^-_p)$ contains all the other literals. Instead, if we revise
the initial theory with $\neg p$ the theory $D^*_{\neg p}$ where
$BS^{+\partial}(D^*_{\neg p}) = \{\neg p, q\}$ and $BS^{-\partial}(D^*_{\neg
p}) = \{p, \neg q\}$ is obtained. The intersections between the revised theory
and the initial one are $BS^{+\partial}(D^*_{\neg p}) \cap BS^{+\partial}(D) =
\{q\}$ and $BS^{-\partial}(D^*_{\neg p}) \cap BS^{-\partial}(D) = \{\neg q\}$.
\end{example}
Again, the main reason for the failure of the Harper Identity resides in the
non-monotonic nature of Defeasible Logic, where, in general it is not possible
to control the consequences of a given formula.

\medskip 

In this section we have provided an interpretation of the AGM postulates for
expansion, contraction and revision in terms of our canonical cases and the
operations that are possible when the changes operate only on the superiority
relation. 

We believe that the contribution of this section is multi-fold.
First of all, the definition of our canonical cases offer a more precise
formal understanding of the intuition of the various operations. Second, we
reconstructed the postulates for the canonical cases\footnote{Notice that
while the main analysis in this paper is specific to revision of the
superiority relation of Defeasible Logic, the definition of the canonical
cases does not depend on it, and it can be applied in a much broader context.
For example the canonical case from $+\partial p$ to $+\partial\neg p$ can be
understood as ``how do we modify a theory such that before the revision a
formula holds, and after the revision the opposite holds?''; similarly for the
other canonical cases.} and discuss how to adapt them. The last contribution
of the analysis confirms the outcome of \cite{ki99}, showing that in general
the postulates describing inclusion relationships between belief sets before
and after a revision operation do not hold for Defeasible Logic, and it is
unreasonable to expect that they hold for non-monotonic reasoning in general.

%% file: RelatedWork.tex
\section{Related Work} 
\label{sec:related_work}

As far as we are aware of, the work most closely related to ours is that of
\cite{DBLP:conf/ijcai/InoueS99} where the authors study, given a theory, how
to abduct preference relation to support the derivation of a specific
conclusion. Therefore the problem they address is conceptually
different from what we presented in this paper, given that we focus on modifying
the superiority relation.

Notice that in non-monotonic reasoning, a revision is not necessarily
triggered by inconsistencies. \cite{ki99} investigates revision for Defeasible
Logic and relationships with AGM postulates. While their ultimate aim is similar
to that of the present paper -- i.e., transforming a theory to make a
previously provable (resp.\ non provable) statement, non provable (resp.\
provable) -- the approach is different, and more akin to standard belief
revision. More precisely, revision is achieved by introducing new exceptional
rules. Furthermore, they discuss how to adapt the AGM postulates for
non-monotonic reasoning.


Our work is motivated by legal reasoning, where preference revision is just
one of the aspects of legal interpretation. \cite{aicol2009,kr2010} propose a
Defeasible Logic framework to model extensive and restrictive legal
interpretation. This is achieved by using revision mechanisms on constitutive
rules, where the mechanism is defined to change the strength of existing
constitutive rules. Based on the specific type of norm to modify, they propose
a revision (contraction) operator which modifies the theory by adding
(removing) facts, strict rules, or defeaters, raising the question whether
extensive and restrictive interpretation can be modelled as preference
revision operators. An important aspect of legal interpretation is finding the
legal rules to be applied in a case: in this work we assumed that the relevant
rules have already been discovered, and in case of conflicts, preference
revision can be used to solve them.

Another work, related to revision of Defeasible Logic is that of 
\cite{GovRot:igpl09}, where the key idea is to model revision operators
corresponding to typical changes in the legal domain, specifically, abrogation
and annulment. They show that, typically, belief revision methodologies are
not suitable to changes in theories intended for legal reasoning. Similarly, they show that it is possible to revise theories fully
satisfying the AGM postulates, but then the outcome is totally meaningless
from a legal point of view.

The connection between sceptical non-monotonic formalisms and argumentation is
well known in literature; in \cite{jlc:argumentation}, authors adapt Dung's
argumentation framework \cite{DBLP:conf/iclp/Dung93a,DBLP:journals/ai/Dung95}
to give an argumentation semantics for Defeasible Logic: first, they prove
that Dung's grounded semantics characterises the \emph{ambiguity propagating
DL}; then, they show that the \emph{ambiguity blocking DL} is described with
an alternative notion of Dung's acceptability. The main effort was to establish close connections between defeasible
reasoning and other formulations of non-monotonic reasoning.

Non-monotonic revision through argumentation was also investigated in
\cite{DBLP:conf/aaai/MoguillanskyRFGS08,DBLP:conf/comma/MoguillanskyRFGS10}
using Defeasible Logic Programming (D\textsc{e}LP). They define an argument
revision operator that inserts a new argument into a defeasible logic program
in such a way that this argument ends up undefeated after the revision, thus
warranting its conclusion, where a conclusion $\alpha$ is warranted if there
exists a non-defeated argument supporting it. Despite the meaning
given in this work, their concept of \emph{defeaters} denotes stronger counter-arguments to a
given conclusion based on a set of \emph{preferences} stating which argument
prevails against one other.

Their work suffers from a main drawback: imposing preferences among
arguments (i.e., whole reasoning chains in our framework), instead of single
rules, can lead to a situation when an argument is warranted even if all its sub-arguments are defeated.



D\textsc{e}LP formalism is very similar to Defeasible Logic. Therefore,
techniques proposed in this work can be easily accommodated to join the
framework presented in
\cite{DBLP:conf/aaai/MoguillanskyRFGS08,DBLP:conf/comma/MoguillanskyRFGS10}.

Other works closely related to ours are
\cite{DBLP:journals/jancl/PrakkenS97,DBLP:journals/ijis/Antoniou04,DBLP:journals/jair/Brewka96,DBLP:conf/jelia/Modgil06}.
They propose extensions of an argumentation framework, Defeasible Logic and
Logic Programming, where the superiority relation is dynamically derived from
arguments and rules in given theories. While the details are different for the
various approaches, the underlying idea is the same. For example, in
\cite{DBLP:journals/ijis/Antoniou04}, it is possible to have rules of the form
$r: a \Rightarrow (s > t)$ where $s$ and $t$ are identifiers for rules.
Accordingly, to assert that rule $s$ is stronger than rule $t$ we have to be
able to prove $+\partial a$ and that there are no applicable rules for $\neg (s>t)$. In
addition, the inference rules require that instances of the superiority
relation are provable (e.g., $+\partial(s>t)$) instead of being simply given
(as facts) in $>$, i.e., $(s,t)\in >$. The main difference with these works is
that we investigate general conditions under which it is possible to modify
the superiority relation in order to change the conclusions of a theory, while
they provide specific mechanisms to compute conclusions where the preference
relations are inferred from the context. They do not study which are the
possible ways to revise a theory. For example, if a literal is
$>$-tautological, no matter how we derive instances of the preference
relation, there is no way to prevent its derivation, or to derive its
negation.

In the scenario where the preferences over rules are computed dynamically, one
could argue that it might be possible to encode in the theory the possible
ways in which the superiority relation would behave. The problem with this
approach is the combinatorial explosion of the number of rules required, since
one would have to consider rules with the form $a_{1},\dots,a_{n}\Rightarrow
(r_{i}>r_{j})$ for all possible combinations of literals $a_{k}$ in the
theory, and also for all possible combinations of instances of $>$. In both
cases there is an exponential number of combinations. Among the works
mentioned above, \cite{DBLP:journals/jancl/PrakkenS97} is motivated, as us, by
legal reasoning, and they use rules to encode the legal principles we shortly
discussed in the introduction.


%% file: concl.tex
\section{Conclusions and further work}\label{sec:Conclusions}
Over the years Defeasible Logic has proved to be simple but effective
practical non-monotonic formalism suitable for applications in many areas. Its
sceptical nature allows to have defeasible proofs both for a belief and its
opposite, and still be consistent since, at most, one of them can be finally
proven. Since from its first formulation in \cite{nute}, many theoretical
aspects of Defeasible Logic have been studied: from its proof theory
\cite{tocl} to relationships to logic programming \cite{tplp:embedding}, from
variants of the logic \cite{tocl:inclusion} to its semantics
\cite{jlc:argumentation} and computational properties \cite{complexity}.
Furthermore, several efficient implementations have been developed
\cite{dr-device,dr-prolog,spindle}. Methods to revise, contract, or expand a
defeasible theory were first proposed in \cite{ki99}, where the authors
studied how to revise the belief set of a theory based on introduction of new
rules. The resulting methodology was then compared to the AGM belief revision
framework.

In this work we took a different approach: since, in many situations, a person
cannot change the rules governing a system (a theory) but only the way each
rule interact with the others, it seems straightforward to consider revision
methodologies of Defeasible Logic where derivation rules are considered as ``static'' or
``untouchable'', and the only way to change a theory with respect to a
statement is to modify the relative strength of a rule with respect to another
rule, i.e., to modify the superiority relation of the analysed theory.

Therefore, we presented in Section \ref{sec:DefLogPRDL}
the formalism adopted: eight different types of tagged literals were described
to simplify the categorisation process and, consequently, the revision
calculus. In Section \ref{sec:PrefDefRevisionPRDL}, we introduced three
\textit{canonical cases} of possible revisions and systematically analysed every \textit{canonical instance}. In both sections, we presented
several theoretical results on conditions under which a revision process is
possible.

Upon these theoretical basis, in Section \ref{sec:AGMComparisonPRDL}, we
proposed a systematic comparison between our framework and the AGM postulates.
In there, the three canonical cases were compared to the AGM contraction,
expansion, and revision: for each belief change operator, all the AGM
postulates were rewritten using our terminology, and their validity was
studied in our framework.

The work presented in this paper paves the way to several lines of further
investigation to extend the proposed change methodologies.

The first extension we want to mention is that where we change the status of a
sets of literals instead of a single literal. Studying conditions (supporting
chains, proof tags, and so on) to understand when, and where, it is possible
to change a theory by more than a single literal is not a trivial issue.
Consider the following theory

\begin{example}
$D =$
\setlength{\arraycolsep}{1pt}
\[
\begin{array}{cccccc}
                    &               &   \To_{r_1}   &    \neg b   &   \To_{r_2}   &   q  \\
   \To_{r_3}   &     a        &   \To_{r_4}   &            b    &   \To_{r_5}   &   p  \\
   \wedge     &                &                    &                  &                    &       \\
   \To_{r_6}   & \neg  a.   &    &      &                    &       \\  
\end{array}
\]
It is clear it is not possible to change the initial theory if we want to obtain both $+\partial p$ and $+\partial q$.
\end{example}

The second extension concerns how to limit the scope of the revision
operators. Revision of preferences should not involve minimal defeasible rules. This constraint captures the idea that a rule that wins against all other rules is a basic juridical principle. A similar aspect is that under given circumstances the revision process should not, for at least a subset of
``protected'' pairs, violate the original preferential order. For instance, we
should not revise those preferences that are unquestioned because derived by
commonly accepted principles or explicitly expressed by the legislator, as
discussed in the introduction.

As we have seen in Section~\ref{sec:Norms}, in the legal domain we can
identify several sources for the preference relation. Preference handling in
Defeasible Logic can gain much from typisation of preferences themselves. The
notion of preference type and its algebraic structure has been studied
previously and can be applied directly here \cite{Cristani02}. Analogously,
one of the possible directions of generalisation for the notion of preference
is the notion of partial order, investigated at a combinatorial level by
\cite{Duntsch94} and then studied from a computational viewpoint in
\cite{Cristani04}.

The main aim of the paper was to identify conditions under which revision
based on changes of the superiority relation was possible. Accordingly, the
next important aspect of belief revision is to identify criteria of \emph{minimal change}. It is possible to give alternative definitions of
minimal revision. For example, one notion could be on the cardinality of
instances of the superiority relation, while another one is to consider
minimality with respect to the conclusions derived from a theory. A few
research questions naturally follow: `Are
there conditions on a theory to guarantee that a revision is minimal?', or `Is it possible to compare different minimality criteria?'.

We illustrate some of these issue with the help of the following example.

\begin{example}\label{ex:Concl}
Let $D$ be the following theory
\setlength{\arraycolsep}{1pt}
\[
\begin{array}{cccccccc}
             &          &  \To_{r_1}  &       a  &  \To_{r_2}  &  b  &  \To_{r_{3}}  &  p  \\
             &          &  \To_{r_4}  &  \neg a  &  \To_{r_{5}}  &  \neg b  &            &  \\
\To_{r_{6}}  &       c  &  \To_{r_7}  &       d  &  \To_{r_8}  &  e  &  \To_{r_{9}}  &  p  \\        
\wedge &          &             &          &             &     &  & \\
\To_{r_{10}} &  \neg c  &             &          &             &     &  &
\end{array}
\]
\begin{eqnarray}
>'  &=& \{(r_{6}, r_{10})\} \nonumber\\	
>'' &=& \{(r_{1}, r_{4}), (r_{2}, r_{5})\}. \nonumber
\end{eqnarray}
Superiority relation $>'$ guarantees to change only one preference, but modify the extension of the former theory by five literals ($c, \neg c, d, e,$ and $p$), while $>''$, by adding two preferences, changes only three literals ($a, b$ and $p$).
\end{example}

Finally, the present work provides a further indication that the AGM
postulates are not appropriate for belief revision of non-monotonic reasoning.
Consequently, a natural question is whether there is a set of rational
postulates for this kind of logics. We are sceptical about this endeavor:
there are many different and often incompatible facets of non-monotonic
reasoning, and a set of postulates might satisfy some particular
non-monotonic features but not appropriate for others. For example, as we have
seen in this paper, if we ignore monotonic conclusions (conclusions tagged
with $\pm\Delta$), there are other cases where we cannot guarantee the success
of the revision operation. On the other hand, \cite{ki99} argues that the
success postulate for revision holds if we are allowed to operate on rules
instead of preferences. This example suggests that it might possible to find a
set of postulates, but this would specific to a logic and specific types of
operations. The quest for an alternative set of postulates for revision of
non-monotonic theories is left for future research.
